\documentclass{new_tlp}

\usepackage[utf8]{inputenc}
\usepackage{amsmath}
\usepackage{amssymb}
\usepackage{url}
\usepackage{xspace}
\usepackage{xcolor}
\usepackage{bm}

\usepackage{tikz}

\DeclareMathOperator{\K}{\mathbf{K}}
\DeclareMathOperator{\M}{\mathbf{M}}
\DeclareMathOperator{\LK}{\mathbf{L}}
\newcommand{\sneg}{\text{-}}
\DeclareMathOperator{\Not}{\text{not}}
\DeclareMathOperator{\eNot}{\mathbf{not}}
\DeclareMathOperator{\tor}{\,\text{or}\,}

\newtheorem{definition}{Definition}
\newtheorem{proposition}{Proposition}
\newtheorem{theorem}{Theorem}
\newtheorem{example}{Example}
\newtheorem{corollary}{Corollary}
\newtheorem{property}{Property}
\newtheorem{observation}{Observation}

\makeatletter
\def\@enumerateone{\expandafter
   \list
     \csname label\@enumctr\endcsname
     {\usecounter{\@enumctr}\let\makelabel\makeRRlabel}}
\makeatother

\def\At{\text{\em At}}

\newcommand{\eqdef}{\mathrel{\vbox{\offinterlineskip\ialign{\hfil##\hfil\cr $\scriptscriptstyle\mathrm{def}$\cr \noalign{\kern1pt}$=$\cr \noalign{\kern-0.1pt}}}}}
\def\sI{\mathcal{I}}

\def\At{\text{\em At}}

\newcommand{\C}{\mathcal{C}}

\def\alive{\mathit{alive}}

\def\trigger{\mathit{trigger}}
\def\load{\mathit{load}}
\def\loaded{\mathit{loaded}}

\def\impossible{\mathit{impossible}}

\def\eligible{\mathit{eligible}}
\def\minority{\mathit{minority}}
\def\high{\mathit{high}}
\def\fair{\mathit{fair}}
\def\interview{\mathit{interview}}
\def\appointment{\mathit{appointment}}

\def\Atoms{\mathit{Atoms}}

\def\SM{\text{\rm SM}}

\newcommand\wv{\mathbb{W}}
\newcommand\wvb{\bm{\mathsf{W}}}
\newcommand\wx{\mathbb{X}}
\def\cS{\mathcal{S}}
\newcommand\cset[1]{[#1]}
\newcommand\us{\mathbb{S}}
\newcommand\kdint[2]{\tuple{#2, #1}}

\def\Bodym{\mathit{Body}_{sub}}
\def\Bodyr{\mathit{Body}_{obj}}
\def\Head{\mathit{Head}}
\def\Body{\mathit{Body}}

\def\Bodyrp{\mathit{Body}^+_{obj}}
\def\Bodymp{\mathit{Body}^+_{sub}}

\newcommand{\fF}{F}
\newcommand{\fG}{G}

\newcommand{\falsif}{=\!\!\!| \;}

\newcommand{\tuple}[1]{\ensuremath{\langle #1 \rangle}}
\newcommand{\sset}[1]{\ensuremath{[#1]}}
\newcommand{\set}[1]{\ensuremath{\{#1\}}}
\newcommand{\setm}[2]{\ensuremath{\{\ #1\ |\ #2\ \}}}
\newcommand{\cI}{\mathcal{I}}

\newcommand{\EM}{\mathbf{EM}}
\newcommand{\KEM}{\mathbf{KEM}}

\newcommand{\Mefwv}{\mbox{M85-world}\xspace}
\newcommand{\Sntwv}{\mbox{S92-world}\xspace}

\newcommand{\mike}{\mathit{mike}}
\newcommand{\guess}{\mathit{guess}}
\newcommand{\define}{\mathit{define}}
\newcommand{\test}{\mathit{test}}
 
\begin{document}

\title{Thirty years of Epistemic Specifications}

\author[Jorge Fandinno, Wolfgang Faber and Michael Gelfond]{Jorge Fandinno$^{12}$, Wolfgang Faber$^3$  and Michael Gelfond$^4$\\
  $^1$University of Nebraska Omaha, USA\\
  $^2$University of Potsdam, Germany\\
  \email{jfandinno@unomaha.edu}\\
  $^3$Alpen-Adria-Universität Klagenfurt, Austria\\
  \email{wolfgang.faber@aau.at}\\
  $^4$Texas Tech University, USA\\
  \email{michael.gelfond@ttu.edu}\\
}

\maketitle

\begin{abstract}
The language of \emph{epistemic specifications} and  \emph{epistemic logic programs} extends disjunctive logic programs under the \emph{stable model} semantics with modal constructs called subjective literals.
Using subjective literals, it is possible to check whether a regular literal is true in every or some stable models of the program,
those models, in this context also called \emph{belief sets}, being collected in a set called \emph{world view}.
This allows for representing, within the language, whether some proposition should be understood accordingly to the open or the closed world assumption.
Several attempts for capturing the intuitions underlying the language by means of a formal semantics were given, resulting in a multitude of proposals that makes it difficult to understand the current state of the art.
In this paper, we provide an overview of the inception of the field and the knowledge representation and reasoning tasks it is suitable for.
We also provide a detailed analysis of properties of proposed semantics, and an outlook of challenges to be tackled by future research in the area.
\end{abstract}

\section{Introduction}\label{sec:introduction}

The language of \emph{epistemic specifications}~\cite{gelfond91a,gelpri93,gelfond94} (a.k.a. \emph{epistemic logic programs}),
proposed by Gelfond in 1991, extends disjunctive logic programs (under the \emph{stable model} semantics;~\citeNP{gellif88b,gellif91a}) with modal constructs called \emph{subjective literals}.
The introduction of this extension was originally motivated by the need to correctly represent
incomplete information in programs that have several stable models.
Using subjective literals, it is possible to check whether a regular literal is true in every or some stable models of the program,
those models being collected in a set called \emph{world view}.
This allows for representing, within the language, whether some proposition should be understood accordingly to the open or the closed world assumption.

Unfortunately, as first noticed by Teodor Przymusinski, world views of epistemic specifications
in this original semantics do not always correspond to those intended by the authors.
This was due to the presence of unsupported beliefs.
Attempts to get rid of these unsupported beliefs were first made by Gelfond in~\citeyearNP{gelfond94} and later in~\citeyearNP{gelfond11a}, followed by many other authors who proposed several new semantics attempting to solve this problem. Somewhat complicating matters, there were also proposals for extending or changing the language.
In this paper, we present a summary of the state of the art regarding epistemic specifications.

The structure of the paper is as follows.
In Section~\ref{sec:inception}, an account of the inception of the field is provided.
This also describes the main intuitions underlying epistemic specifications.
We then review the formal details of epistemic specifications and provide an updated view of the ideas behind the original definition of epistemic specifications (Section~\ref{sec:theories}). The language of epistemic specifications is general enough to accommodate the syntax of most proposals in the literature, but we also define the language of epistemic logic programs, which is the fragment treated by most proposals.
In Section~\ref{sec:applications}, we review some of the representation problems that epistemic specifications can help to solve.
We relate these problems with some formal properties that help understanding the behavior for that specific task.
Namely, we revisit the use of epistemic specifications to express integrity constraints over disjunctive databases, informally discussed in Section~\ref{sec:inception}, but including technical details.
Then, we illustrate how epistemic specifications are also useful when we need to reason about all the answer sets of a program at the same time, which is not possible by simply using Answer Set Prolog.
We show also how we can extend the guess-define-and-test methodology from Answer Set Programming to problems that lie on the second level of the polynomial hierarchy: in particular 
we illustrate this methodology showing how epistemic logic programs can be used to find conformant plans.
Finally, we provide an example from cybersecurity.
In Section~\ref{sec:semantics}, we survey the path followed in the attempt to get rid of self-supported beliefs and the different approaches proposed in the literature.
We also show which properties are satisfied (or not) by the various semantics and describe some deeper relationships between some of the semantics.
Section~\ref{sec:ael} puts epistemic specifications in the broader context of Knowledge Representation by studying the relation between epistemic specifications and autoepistemic logics.
Finally, Section~\ref{sec:conclusions} concludes the paper and  presents some challenges for the future.
 \section{The Inception}\label{sec:inception}

The idea of epistemic specifications was initially suggested
in three consecutive papers \cite{gelfond91a,gelfondP93,gelfond94}.
This work was a part of the larger research program, originated by John
McCarthy and others in the late fifties.
The program aimed to develop knowledge representation
languages capable of clear and succinct formalization of substantial parts
of commonsense knowledge and commonsense reasoning methods.
A substantial step in this direction was made
by~\mbox{\citeN{gellif91a}},
who extended the language of ``classical'' logic programming with disjunction
and classical negation.\!\!\footnote{Often referred to as strong or explicit negation.}
The new language allowed reasoning with some forms
of incomplete information. For instance, for a program with one answer set, say $S$,
a statement ``the truth or falsity of $p$ is \emph{unknown}'' can be expressed
in Answer Set Prolog as
\begin{gather*}
\Not p,\, \Not \sneg p
\end{gather*}
where~``$\Not$'' and~``$\sneg$'' respectively stand for default and classical negation.
The Closed World Assumption (CWA; \citeNP{reiter78a}), stating that
``$p(X)$ is false unless there is a reason to believe it to be
true'' has the form
\begin{gather}
\sneg p(X) \leftarrow \Not p(X),
	\label{eq:cwa.Not}
\end{gather}
These representations,
however, do not work for programs with multiple answer sets. The main goal of epistemic
specifications was to address this deficiency.
As stated by~\mbox{\citeN{gelfond91a}}, we wanted to ``expand the syntax and semantics of logic programs
and deductive databases to allow for the correct representation of incomplete information
in the presence of multiple extensions.'' The main idea was to expand the syntax and semantics
of Answer Set Prolog by modal operators $\K$ and $\M$ where $\K F $ holds if $F$ is true in all
answer sets of a program and $\M F$ holds if $F$ is true in at least one answer set.
In this notation
$$\Not \K p,\, \Not\K\sneg p$$
would correspond to ``the truth value of $p$
is unknown'' even in the presence of multiple answer sets; the CWA for a relation $p$
could be expressed as
$$\sneg p(X) \leftarrow \Not \M p(X).$$
In a language containing object constants $a$, $b$ and $c$
this rule, combined with a rule
\begin{gather*}
p(a) \tor p(b)
\end{gather*}
would produce answer \emph{No} to
a query $p(c)?$, but remain undecided about query $p(a)?$.
The same behavior will, of course,
be produced by the original  representation~\eqref{eq:cwa.Not} of CWA.
However, for a more complex query, say
\begin{gather*}
(\sneg p(a) \tor \sneg p(b))?
\end{gather*}
the behaviors differ: the
former answers the query by \emph{Yes}, while the latter remains undecided.
This was intended -- we wanted a form of CWA not applicable to undecided disjuncts.

The new features of epistemic specifications were not limited to modal operators.
Rules were allowed to contain more general formulas (most importantly existential quantifiers).
In addition to usual (Herbrand) objects constants, there were also so called \emph{generic constants} used to
refer to \emph{unnamed} objects. The former were defined by atoms of the form
$h(c)$ (where $h$ stands for Herbrand), listing all the named objects of the domain, together with the rule
\begin{gather*}
\sneg h(X) \leftarrow \mbox{not } h(X).
\end{gather*}
This separation between named and unnamed objects allowed
representation of information which would be difficult to express otherwise. In particular, it was used
to remove the Domain Closure Assumption from the semantics of logic programs.
Instead the assumption, which states that
``all objects in the domain of discourse described by a program $\Pi$ have names in the signature of $\Pi$'',
could be expressed by the constraint
$$ \leftarrow \exists X\, \sneg h(X).$$
Existential quantifiers combined with modal operators were shown to be instrumental in expressing various forms of
constraints understood as statements about the content of the knowledge base as opposed to
statements about the world \cite{Reiter92}. Consider, for instance, knowledge base $T$
\begin{flalign}
&\begin{aligned}
&\mathit{h(bob)} \ \ \ \ \ \ \mathit{h(mary)}\\
&\mathit{teach(bob,java)}\\
&\mathit{teach(staff,python)}\\
&\mathit{teach(bob,ai)} \tor \mathit{teach(mary,ai)}.
\end{aligned}&
	\label{eq:disjuntive.db}
\end{flalign}

\medskip\noindent
where $\mathit{bob}$ and $\mathit{mary}$ are professors in the department, and $\mathit{staff}$ refers to a professor yet to be hired.
Then a constraint
\begin{gather}
\exists X \K (h(X) \wedge \mathit{teach}(X,C))
	\label{eq:constraint.exists.K}
\end{gather}
is satisfied by $C=\mathit{java}$ (which is taught by Bob).
A weaker constraint
\begin{gather}
\exists X \K \mathit{teach}(X,C)
	\label{eq:constraint.exists.K2}
\end{gather}
is satisfied by $C=\mathit{java}$ and by $C=\mathit{python}$;
\begin{gather}
\K \exists X \mathit{teach}(X,C)
	\label{eq:constraint.K.exists}
\end{gather}
is satisfied by all three classes ($\mathit{java}$, $\mathit{python}$, $\mathit{ai}$).

\medskip

The semantics of the language was similar to that of Answer Set Prolog. In both cases
a program was viewed as a specification of sets of beliefs that could  be held by
a rational reasoner associated with the program. But, while in Answer Set Prolog rules
constrain the formation of each set of beliefs (i.e. each answer set)
independently from others, in epistemic specifications restrictions are put also
on the relationship between such sets. This intuition led to the notion of a
\emph{world view} -- a collection of answer sets formed simultaneously by a
rational agent to satisfy the program's rules.
The key technical problem, as in the
semantics of~Answer Set Prolog, was to find the proper definition of a reduct capturing
rationality of the agent.
While the original paper \cite{gelfond91a} had an egregious
error in this definition (to the best of the author's recollection introduced at the last
moment in the attempt to satisfy time and space requirements of the conference)
other two papers \cite{gelfondP93,gelfond94} had a definition believed to be reasonable. It soon became
clear, however, that this belief was unjustified. To see the reason, consider an
epistemic specification consisting of one rule:
\begin{gather}
p \leftarrow \K p.
	\label{es:self-supported}
\end{gather}
To the authors' surprise it was noticed that, according to the definition proposed by~\citeN{gelfondP93} and latter used by~\citeN{gelfond94}, it has two world views:
$[\{ \, \}]$ and $[\{p\}]$. The latter contains the unsupported belief $p$ and
is clearly unintended. According to the rationality principle, which serves as
the foundation of the semantics of Answer Set Prolog, an agent is not supposed to believe
anything that it is not forced to believe, which is the case in the second world view.
For some time Gelfond had been trying to modify the definition but, after a few years
of failure, gave up on the idea.

In~\citeyearNP{gelfond11a}, Gelfond gave yet another attempt to modify this definition and was soon joined by many other authors in this attempt, with several new semantics attempting to solve this problem~\cite{kawabagezh15,faheir15a,sheeit16,sheeit17a,cafafa19a,irazu20}.

Admittedly, rule~\eqref{es:self-supported} is unlikely to be written by a programmer.
It is, however, used here to distill a phenomenon that can occur as a result of more complex and reasonable rules.
As an example consider the following rule
\begin{gather*}
r(Y) \leftarrow \K r(X) \wedge edge(X,Y)
\end{gather*}
saying that if in a world view $\wv$, property~$r$ is known to be true in state $X$, and $Y$ is a
successor of $X$ then $r(Y)$ must be included in every belief set of $\wv$.
When this rule is combined with facts~$edge(a,b)$, $edge(c,d)$ and~$edge(d,c)$ representing a graph and the fact~$r(a)$ stating that the property~$r$ is satisfied in state~$a$,
we can observe that the resultant program suffers from the same problem as~\eqref{es:self-supported}. \section{Epistemic Theories}
\label{sec:theories}

In this section,
we review the syntax of epistemic theories.
We present a language flexible enough to relate all the approaches that we will study in this paper.
We start by introducing epistemic theories in a general way and later we review a specific subset corresponding closely to the syntax of logic programs.

\subsection{General Syntax}

The language of epistemic specifications is that of first-order modal logic~\cite{fitting2012first}
extended with explicit negation.
We follow the convention of the literature on epistemic specifications for writing modal operations.
That is, symbols~$\K$ and~$\M$ are used in place of~$\Box$ and~$\Diamond$, respectively.
Terms and atoms are defined as usual in first-order (non-modal) logic.
Formulas are defined according to the following grammar:
\[
\fF \ ::= \ 
\bot \,\mid\,
\top \,\mid\,
a \,\mid\, 
\sneg \fF \,\mid\, 
\fF_1 \wedge \fF_2 
\,\mid\, \fF_1 \vee \fF_2
\,\mid\, \fF_1 \leftarrow \fF_2
\,\mid\, \K \fF
\,\mid\, \M \fF
\,\mid\, \exists x \,  \fF
\,\mid\, \forall x \,  \fF
	\]
with~$a \in \At$ an atom and~$x$ an object variable.
We assume that ${\Not \fF}$ is an abbreviation for~${\bot\leftarrow\fF}$.
We call~``$\sneg$'' \emph{explicit negation} and~``$\Not$'' \emph{default negation}.
An occurrence of a variable $x$ in a formula~$\fF$ is bound if it belongs to a subformula of~$\fF$ that has the forms~$\forall x \, \fG$ or~$\exists x \, \fG$; otherwise it is free.
A \emph{sentence} is a formula without free variables.
An \emph{(epistemic) theory}~$\Gamma$ is a set of sentences.
We sometimes write formulas with free variables that should be understood as their universal closure.

An \emph{explicit literal} is either an atom or a formula of the form~$\sneg a$ with $a$ being an atom.
Terms, atoms, explicit literals and formulas not containing variables are called \emph{ground}.

\subsection{Monotonic semantics}
\label{sec:theories.semantics}

We introduce here two monotonic semantics for epistemic theories that will be instrumental in defining the non-monotonic semantics in the next section.
The semantics discussed here coincide with those of modal logics S5 and KD45 extended with strong negation~\cite{nelson1949,vakarelov1977notes}.

An \emph{interpretation} is a set of ground explicit literals~$I$ such that either~${a \notin I}$ or~${\sneg a \notin I}$ for every atom~$a$.
An \emph{epistemic interpretation}~$\wv$
is a non-empty set of interpretations.
A \emph{belief interpretation} $\sI = \tuple{\wv,I}$ is a pair where~$I$ is a propositional interpretation and~$\wv$ is an epistemic interpretation.
We write~${\tuple{\wv,I} \models \fF}$ to represent that a belief interpretation
${\tuple{\wv,I}}$ \emph{satisfies} a sentence~$\fF$
and~${\tuple{\wv,I} \falsif \fF}$ to represent that a belief interpretation ${\tuple{\wv,I}}$ \emph{falsifies} a sentence~$\fF$.
These two relation are defined according to the following mutually recursive conditions:
\begin{enumerate}
\item $\tuple{\wv,I} \models \top$;
\item $\tuple{\wv,I} \models a$ if $a \in I$, for any atom $a \in \At$:

\item $\tuple{\wv,I} \models \fF \wedge \fG$ if $\tuple{\wv,I} \models \fF$ and $\tuple{\wv,I} \models \fG$;

\item $\tuple{\wv,I} \models \fF \vee \fG$ if $\tuple{\wv,I} \models \fF$ or $\tuple{\wv,I} \models \fG$;

\item $\tuple{\wv,I} \models \fF \leftarrow \fG$ if 
$\tuple{\wv,I} \models \fF$ or $\tuple{\wv,I} \not\models \fG$

\item $\tuple{\wv,I} \models \exists x \, \fF(x)$ if $\tuple{\wv,I} \models \fF(t)$ for some ground term~$t$;

\item $\tuple{\wv,I} \models \forall x \, \fF(x)$ if $\tuple{\wv,I} \models \fF(t)$ for all ground terms~$t$;

\item $\tuple{\wv,I} \models \K \fF$ if $\tuple{\wv,I'} \models \fF$ for all $I' \in \wv$;

\item $\tuple{\wv,I} \models \M \fF$ if $\tuple{\wv,I'} \models \fF$ for some $I' \in \wv$;

\item $\tuple{\wv,I} \models \sneg \fF$ if $\tuple{\wv,I} \falsif \fF$;

\vspace{5pt}

\item $\tuple{\wv,I} \falsif \bot$;
\item $\tuple{\wv,I} \falsif a$ if $\sneg a \in I$, for any atom $a \in \At$:

\item $\tuple{\wv,I} \falsif \fF \wedge \fG$ if $\tuple{\wv,I} \falsif \fF$ or $\tuple{\wv,I} \falsif \fG$;

\item $\tuple{\wv,I} \falsif \fF \vee \fG$ if $\tuple{\wv,I} \falsif \fF$ and $\tuple{\wv,I} \falsif \fG$;

\item $\tuple{\wv,I} \falsif \fF \leftarrow \fG$ if 
$\tuple{\wv,I}  \falsif \fF$ and $\tuple{\wv,I} \models \fG$

\item $\tuple{\wv,I} \falsif \exists x \, \fF(x)$ if $\tuple{\wv,I} \falsif \fF(t)$ for all ground terms~$t$;

\item $\tuple{\wv,I} \falsif \forall x \, \fF(x)$ if $\tuple{\wv,I} \falsif \fF(t)$ for some ground term~$t$;

\item $\tuple{\wv,I} \falsif \K \fF$ if $\tuple{\wv,I'} \falsif \fF$ for all $I' \in \wv$;

\item $\tuple{\wv,I} \falsif \M \fF$ if $\tuple{\wv,I'} \falsif \fF$ for some $I' \in \wv$; and

\item $\tuple{\wv,I} \falsif \sneg \fF$ if $\tuple{\wv,I} \models \fF$.
\end{enumerate}
A belief interpretation~$\tuple{\wv,I}$ that satisfies a formula is called a~\emph{belief model}.
An epistemic interpretation~$\wv$
\emph{satisfies} a formula~$\fF$,
in symbols~$\wv \models \fF$ if~$\tuple{\wv,I} \models \fF$ for all~$I \in \wv$.
In this case, $\wv$ is also called an \emph{epistemic model} of~$\fF$.
Belief and epistemic models defined in this way correspond to models in modal logics~KD45 and~S5, respectively.
As mentioned above, these modal logics are extended here with strong negation (called here explicit negation).

Formulas not containing modal operators are called \emph{objective}.
Formulas in which all atoms are in the scope of modal operators are called \emph{subjective}.
A theory is called \emph{objective} or \emph{subjective} if all its formulas are objective or subjective, respectively.
For an objective formula~$\fF$, the component~$\wv$ is irrelevant.
Therefore, we abbreviate~$\tuple{\wv,I} \models \fF$ as~$I \models \fF$.

\subsection{Nonmonotonic semantics}
\label{sec:original}
We provide now a non-monotonic semantics for epistemic theories.
This semantics is a conservative extension of Answer Set Prolog.
As mentioned earlier,
the initial work focused on a restricted language syntax that did not allow arbitrary formulas as described in Section~\ref{sec:theories}.
\mbox{\citeN{truszczynski11b}} allowed arbitrary propositional formulas, but did not include first-order constructs, such as quantifiers.
However, their ideas apply directly to the language presented above by considering a definition of \emph{stable models} that covers arbitrary objective formulas.
For that definition we rely on quantified equilibrium logic~\cite{peaval06a} with explicit negation~\cite{agcafapepevi19b}.

Given an objective theory~$\Gamma$, by~$\SM[\Gamma]$, we denote the set of interpretations that are \emph{answer sets} (or \emph{stable models}) of~$\Gamma$~(see~\ref{sec:qel} for a formal definition).
With this notation, we can immediately provide an answer set based semantics to arbitrary epistemic theories.

\begin{definition}[G94-reduct]
\label{eq:G94-reduct}
The \emph{G94-reduct} of a theory~$\Gamma$ 
with respect to an epistemic interpretation~$\wv$, written~$\Gamma^\wv$, is obtained by replacing each maximal subformula~$\fF$ of the forms $\K \fG$ and $\M \fG$ by $\top$, if $\wv \models \fF$; or by $\bot$, otherwise.
\end{definition}

\begin{definition}[G94-world view]
\label{def:G94-world.view}
An epistemic interpretation~$\wv$ is called a \emph{G94-world view} of a theory~$\Gamma$
if~${\wv = \SM[\Gamma^\wv ]}$.
\end{definition}

\begin{samepage}
\begin{definition}[$\cS$-Belief set]
\label{def:belief set}
Given a semantics~$\cS$,
an interpretation~$I$ is called an \emph{$\cS$-belief set}
of a theory~$\Gamma$
if there is an $\cS$-world view~$\wv$ of~$\Gamma$ with~$I \in \wv$.
\end{definition}
\end{samepage}

Definition~\ref{def:belief set} is stated in a general way, so it can be applied to different semantics provided that they give a definition of~$\cS$-world views.
In particular, we get the definition of G94-belief sets by replacing~$\cS$ by G94.
This kind of parametrized definition is useful to accommodate different semantics that we review in the following sections.

\subsection{Epistemic Logic Programs}

From a Knowledge Representation point of view,
it is interesting to focus on a particular class of theories that have the form of logic programs with modal operators.
Formally,
an \emph{objective literal} $\ell$ is either an explicit literal, that is $\ell \in \At \cup \{\sneg a \mid a \in \At\}$, a truth constant\footnote{For a simpler description of program transformations, we allow truth constants where $\top$ denotes true and $\bot$ denotes false.}, that is \mbox{$\ell \in \{\top,\bot\}$}, or an explicit literal preceded by one or two occurrences of default negation,
that is \mbox{$l = \Not \ell$} or \mbox{$l = \Not\Not \ell$}.
A \emph{subjective literal} is an expression of the forms
$\K l$, $\M l$,
$\Not \K l$, $\Not \M l$,
$\Not\Not \K l$ or $\Not\Not \M l$
for any objective literal~$l$.
A \emph{literal} is either an objective or a subjective literal.
A \emph{rule} $r$ is an expression of the form:
\begin{gather}
l_1 \tor \dots \tor l_m \leftarrow L_1, \dots, L_n
	\label{eq:rule}
\end{gather}
with $m\geq 0$ and $n\geq 0$, where each $l_i$ is an objective literal and each $L_j$ a literal.
The left hand disjunction of \eqref{eq:rule} is called the rule \emph{head} and it is abbreviated as $\Head(r)$.
The right hand side of \eqref{eq:rule} is called the rule \emph{body} and it is abbreviated as $\Body(r)$.
An \emph{(epistemic) logic program}~$\Pi$ is a set of rules of the form~\eqref{eq:rule}.

We identify each rule of the form of~\eqref{eq:rule}
with the universal closure of the formula
\begin{gather}
l_1 \vee \dots \vee l_m \leftarrow L_1 \wedge \dots \wedge L_n
	\label{eq:rule.formula}
\end{gather}
When~$m=0$, we assume the head of the rule to be~$\bot$.
We also identify each logic program with a theory containing a formula as above for each rule in the program.
Accordingly, we immediately obtain a definition for the G94-world views of an epistemic logic program using Definition~\ref{def:G94-world.view}.
 \section{Epistemic specifications for Knowledge Representation}
\label{sec:applications}

In this section, we review some of the potential applications of epistemic specifications for knowledge representation.
Namely, we revisit the use of epistemic specifications to express integrity constraints over disjunctive databases informally discussed in Section~\ref{sec:inception}.
Then, we illustrate how epistemic specifications are also useful when we need to reason about all the answer sets of an objective program at the same time, which is usually not possible using Answer Set Prolog itself.
Alongside these two applications, we also review two formal properties (called \emph{subjective constraint monotonicity} and \emph{epistemic splitting}) that shed some light on the reasons why epistemic specifications are useful for these two classes of problems.
These properties are also used in the forthcoming sections to compare different semantics.
We then show how we can extend the guess-define-and-test methodology of Answer Set Programming to problems that lie on the second level of the polynomial hierarchy: in particular 
we illustrate this methodology showing how epistemic logic programs can be used to find conformant plans.
This methodology is also based on the aforementioned two properties: subjective constraint monotonicity and epistemic splitting.
Finally, we also sketch a potential application in cybersecurity.

\subsection{Integrity Constraints}

As mentioned in the introduction,
one of the initial motivations for epistemic specifications was to express various forms of constraints about the knowledge of disjunctive databases.
We show here how the above semantics allow us to represent constraints without free-variables, called here \emph{integrity constraints}.
Let us now formalize some of the intuitions mentioned there.

\begin{definition}
An \emph{epistemic specification}
is a pair~$E=\tuple{\Gamma,\C}$
where~$\Gamma$ is an epistemic theory and~$\C$ is a subjective theory
whose sentences are called \emph{integrity constraints}.
An \emph{$\cS$-world view} of~$E$ is an $\cS$-world view~$\wv$ of~$\Gamma$ such that~$\wv \models \C$.
\end{definition}

If we consider now a program containing the knowledge base~\eqref{eq:disjuntive.db},
we can see that such program has a unique world view containing two belief sets:
\begin{gather*}
A \cup \{ \mathit{teach(bob,ai)} \}
\hspace{2cm}
A \cup \{ \mathit{teach(mary,ai)} \}
\end{gather*}
with~$A = \{ \mathit{h(bob)}, \mathit{h(mary)}, \mathit{teach(bob,java)}, \mathit{teach(staff,python)} \}$ being common to both belief sets.
It is easy to see that
formula~${h(\mathit{bob}) \wedge teach(\mathit{bob},\mathit{java})}$
is satisfied by both belief sets.
This implies that
the unique world view of this program
satisfies the formula
${\K (h(\mathit{bob}) \wedge teach(\mathit{bob},\mathit{java}))}$
and, as a result,
also
${\exists X \K (h(X) \wedge teach(X,\mathit{java}))}$.
On the other hand,
neither
${\mathit{teach(bob,ai)}}$
nor
${\mathit{teach(mary,ai)}}$
are satisfied by both belief sets
and as a result,
sentence
${\exists X \K (h(X) \wedge teach(X,\mathit{ai}))}$
is not satisfied by the unique world view.
This implies that the
program does not satisfy the universal closure of constraint~\eqref{eq:constraint.exists.K}, that is, the sentence:
\begin{gather}
\forall C \, \exists X \K (h(X) \wedge \mathit{teach}(X,C))
	\label{eq:constraint.exists.K.sentence}
\end{gather}
Similarly, we can see that the universal closure of constraint~\eqref{eq:constraint.exists.K2}
is not satisfied either, but the universal closure of~\eqref{eq:constraint.K.exists} is.
Note that in each belief set there is someone teaching each of the subjects,
even if that person may vary between belief sets (in the case of $\mathit{ai}$) or may be unknown (in the case of $\mathit{python}$).

An interesting property of some semantics is that integrity constraints can be fully integrated into a single theory, while other semantics do not allow for this.
This property was called~\emph{subjective constraint monotonicity} by~\citeN{cafafa19b}.

\begin{property}[Subjective constraint monotonicity]
\label{property:constraint.monotonicity}
A semantics~$\cS$ is said to satisfy \emph{subjective constraint monotonicity} if, for any epistemic specification~$E=\tuple{\Gamma,\C}$, 
an epistemic interpretation $\wv$ is a $\cS$-world view of $E$ iff $\wv$ is a $\cS$-world view of
$\Gamma \cup \{ \bot \leftarrow \Not\varphi \mid \varphi \in \C \}$.
\end{property}

This property is analogous to the monotonicity of constraints in Answer Set Prolog.
Recall that an interpretation is an answer set of a program iff it satisfies all its constraints and is an answer set of the rest of the program.
Similarly, subjective constraint monotonicity allows us to work simply with a single theory (resp. logic program), instead of giving a special treatment to constraints.
It also ensures that certain intuitions from Answer Set Prolog are carried to epistemic logic programs.

Note that Property~\ref{property:constraint.monotonicity} is enunciated in a semantics-dependent way (depends on the semantics~$\cS$ selecting some $\cS$-world views),
so it can be applied to alternative semantics.
With respect to the semantics corresponding to Definition~\ref{def:G94-world.view},
\citeN{cafafa19b} show that Property~\ref{property:constraint.monotonicity} is satisfied for ground theories.
It is not difficult to see that this property is also satisfied for \mbox{non-ground} ones.
We discuss it in the context of other semantics below.
As we point out, some semantics satisfy this property and others do not (see Table~\ref{table:summary} in page~\pageref{table:summary} for a quick overview).

 \subsection{Reasoning about incomplete knowledge}
\label{sec:splitting}

Beyond expressing integrity constraints about the knowledge implied by a database,
an interesting feature of epistemic specifications is their ability to deduce new information about the  knowledge in the database.
To illustrate this claim, consider the following example introduced by~\citeN{gelfond94}.

\begin{samepage}
\begin{example}\label{ex:college}
A given college uses the following set of rules to decide whether a student $X$ is eligible for a scholarship:
\begin{eqnarray}
\eligible(X) & \leftarrow & \high(X)   \label{ex:college.1} \\
\eligible(X) & \leftarrow & \minority(X),\, \fair(X) \label{ex:college.2}\\
\sneg\eligible(X) & \leftarrow & \sneg\fair(X),\, \sneg\high(X) \label{ex:college.3}
\end{eqnarray}
Here, $\high(X)$ and $\fair(X)$ refer to the grades of student~$X$.
We want to encode the additional college criterion
``\emph{The students whose eligibility is not determined by the college rules should be interviewed by the scholarship committee}'' 
as another rule in the program.
\end{example}
\end{samepage}

\noindent The interesting issue is that deciding whether $\eligible(X)$ ``\emph{can be determined}'' requires reasoning about all the stable models of the program at the same time.
For instance, if the only available information for some student $mike$ is the disjunction
\begin{gather}
\fair(\mike) \tor \high(\mike) \label{ex:college.4}
\end{gather}
we get that program \mbox{$\{\, \eqref{ex:college.1} \text{\,-\,} \eqref{ex:college.4}\,\}$} has a unique world view containing the following two belief sets:
\begin{gather}
\{\, \high(\mike), \eligible(\mike) \,\}
	\label{f:sm1.pre}
\\
\{\, \fair(\mike) \,\}
	\label{f:sm2.pre}
\end{gather}
so $\eligible(\mike)$ cannot be determined and an interview should follow.
If we are interested only in querying $\K\eligible(\mike)$ or $\M\eligible(\mike)$, we can do it inside standard logic programming.
For instance, the addition of constraint:
\begin{eqnarray*}
\bot \leftarrow \eligible(\mike)
\end{eqnarray*}
allows us to decide if $\eligible(\mike)$ is a consequence of all answer set of the original program by just checking that the resulting program has no stable model.
In such case, we can also conclude that $\K\eligible(\mike)$ is a consequence of the program.
The difficulty comes when we try to \emph{derive} new information from that knowledge.
Rule
\begin{gather}
\interview(X) \leftarrow \Not\K \eligible(X),\,
\Not\K \sneg \eligible(X) 
\label{ex:college.5}
\end{gather}
precisely allows us to derive that $\interview(X)$ needs to hold for every student~$X$ for whom neither $\eligible(X)$ nor $\sneg eligible(X)$ are satisfied in all belief sets of \mbox{$\{\, \eqref{ex:college.1} \text{\,-\,} \eqref{ex:college.4}\,\}$}.
If we now consider the program
\mbox{$\{\, \eqref{ex:college.1} \text{\,-\,} \eqref{ex:college.4}, \, \eqref{ex:college.5}\,\}$},
we can see that this program has a unique world view containing
the following two belief sets:
\begin{eqnarray}
&\{\, \fair(\mike),\interview(\mike)\,\}& \label{f:sm1}\\
&\{\, \high(\mike), \eligible(\mike),\interview(\mike)\,\} &\label{f:sm2}
\end{eqnarray}
The intuition behind the reasoning process followed in this example
relies on a kind of reasoning by layers.
First, we compute the world views of the first layer~\mbox{$\{\eqref{ex:college.1} \text{\,-\,} \eqref{ex:college.4}\}$};
then, the second layer inspects the world views of the first layer through subjective formulas and derives new information.
We can also extend this example with a third layer that uses the knowledge about $\interview$ to derive further information,
for instance, by including the rule:
\begin{eqnarray}
\appointment(X) \leftarrow  \K \interview(X) \label{ex:college.6}
\end{eqnarray}
The two belief sets of program~\mbox{$\set{\eqref{ex:college.1} \text{\,-\,} \eqref{ex:college.4}, \eqref{ex:college.5}}$} contain $\interview(\mike)$ and, as a result,
we may expect that $\appointment(\mike)$ should be added to both belief sets
of program~\mbox{$\set{\eqref{ex:college.1} \text{\,-\,} \eqref{ex:college.4}, \eqref{ex:college.5}, \eqref{ex:college.6} }$}.
Indeed,
the unique world view of this program contains the two belief sets resulting from adding
$\appointment(\mike)$ to~\eqref{f:sm1} and~\eqref{f:sm2}.

This kind of reasoning was formalized in the form of a \emph{splitting property} by~\citeANP{cafafa19b}~\citeyear{cafafa19b,cafafa21a}.
This property resembles the \emph{splitting theorem} for Answer Set Prolog~\cite{liftur94a}.
It is worth noting that this splitting property was stated only for ground programs.
However, it directly extends to non-ground programs without quantifiers, by understanding each of them as the ground program obtained by replacing all variables by all possible object constants.
It is still an open issue to generalize this property to arbitrary theories containing quantifiers.

We introduce this property now,
but we need the following notation first.
Given a ground rule~$r$ of the form~\eqref{eq:rule},
by~$\Atoms(r)$ we denote the set of all atoms occurring in~$r$.
By~$\Bodyr(r)$ we denote the set of all atoms occurring in objective literals in the body of~$r$.
By abuse of notation we also use $\Head(r)$ to denote the set of all atoms occurring in the head of~$r$.

\begin{definition}[Epistemic splitting set]\label{def:splitting}
A set of ground atoms \mbox{$U \subseteq \At$} is said to be an \emph{epistemic splitting set} of a ground program~$\Pi$ if for any rule $r$ in $\Pi$ one of the following conditions hold:
\begin{enumerate}
\item $\Atoms(r) \subseteq U$,
    \label{item:1:def:splitting}
\item 
$(\Bodyr(r) \cup \Head(r)) \cap U = \emptyset$.
\label{item:3:def:splitting}\end{enumerate}
We define a \emph{splitting} of 
$\Pi$ as a pair $\tuple{B_U(\Pi),T_U(\Pi)}$ satisfying $B_U(\Pi) \cap T_U(\Pi) = \emptyset$ and $B_U(\Pi) \cup T_U(\Pi) = \Pi$,
and also that all rules in $B_U(\Pi)$ satisfy (i) and all rules in $T_U(\Pi)$ satisfy (ii).
\end{definition}

We also need to introduce a variation of the subjective reduct in Definition~\ref{eq:G94-reduct} that is restricted to a particular set of atoms.

\begin{definition}
\label{eq:es-reduct.signature}
The \emph{subjective reduct} of a ground program $\Pi$ with respect to an epistemic interpretations~$\wv$ and a signature $U \subseteq \At$,
written $\Pi^\wv_U$, is obtained by replacing each subjective literal $L$ with $\Atoms(L) \subseteq U$ by~$\top$ if $\wv \models L$ or by~$\bot$ otherwise.
\end{definition}

It is easy to see that, when~$U=\At$,
Definition~\ref{eq:es-reduct.signature} coincides with  Definition~\ref{eq:G94-reduct}.
Given an epistemic splitting set~$U$ for a program~$\Pi$ and an epistemic interpretation~$\wv$, we define $E_U(\Pi,\wv) \eqdef T_U(\Pi)^\wv_U$, that is, we make the subjective reduct of the top with respect to $\wv$ and signature $U$.

\begin{definition}
A pair $\tuple{\wv_b,\wv_t}$ is said to be an $\cS$-\emph{solution} of a ground program~$\Pi$ with respect to an epistemic splitting set $U$ if $\wv_b$ is an $\cS$-world view of $B_U(\Pi)$ and $\wv_t$ is an $\cS$-world view of $E_U(\Pi,\wv_b)$.
\end{definition}

The following operation allows reconstructing the world view of the whole program from the world views of its parts:
$$\wv_b \sqcup \wv_t \ \  = \ \ \setm{I_b \cup I_t}{ I_b \in \wv_b  \text{ and } I_t \in \wv_t  }$$

\begin{property}[Epistemic splitting]
\label{property:epistemic.splitting}
A semantics $\cS$ satisfies \emph{epistemic splitting} if for any epistemic splitting set $U$ of any ground program $\Pi$, epistemic interpretation~$\wv$ is an $\cS$-world view of $\Pi$ iff there is an $\cS$-solution $\tuple{\wv_b,\wv_t}$ of~$\Pi$ with respect to $U$ such that
$\wv=\wv_b \sqcup \wv_t$.
\end{property}

As with subjective constraint monotonicity,
this property is also stated in a semantics-dependent way,
so we can study its applicability to other semantics reviewed later.
In particular,
the semantics described above does satisfy this property~\cite[Main Theorem]{cafafa19b}.
Interestingly,
every semantics satisfying epistemic splitting also satisfies subjective constraint monotonicity (Property~\ref{property:constraint.monotonicity}; for a proof of this result see the paper by~\citeNP[Theorem~3]{cafafa19b}).
A different notion of splitting in the context of this semantics was first studied by~\citeN{watson00}.

Getting back to our running example,
we can see that the set $U$, consisting of atoms $\high(\mike), \fair(\mike),$ $\eligible(\mike), \minority(\mike)$,
is an epistemic splitting set that divides the program~\mbox{$\{\, \eqref{ex:college.1} \text{\,-\,} \eqref{ex:college.4}, \eqref{ex:college.5} \,\}$}
into a bottom 
\mbox{$\{\,\eqref{ex:college.1} \text{\,-\,} \eqref{ex:college.4}\,\}$}
and top part
\mbox{$\{\,\eqref{ex:college.5}\,\}$}.
The bottom part is an objective program, without epistemic operators,
which has a unique world view
$\wv_b=
[ \eqref{f:sm1.pre}, \eqref{f:sm2.pre} ]$.
The corresponding simplification of the top contains (after grounding) the single rule
\begin{gather}
\interview(\mike) \leftarrow \Not \bot, \Not \bot
\end{gather}
Again, this program is objective and its unique world view is $\wv_t=[\{interview(mike)\}]$.
Now it is easy to see how epistemic splitting guarantees that program
\mbox{$\{\,\eqref{ex:college.1} \text{\,-\,} \eqref{ex:college.4}, \eqref{ex:college.5}\,\}$}
has a unique world view~$\wv_b \sqcup \wv_t = [\eqref{f:sm1}, \eqref{f:sm2}]$.
We can recursively apply this reasoning to program~\mbox{$\{\,\eqref{ex:college.1} \text{\,-\,} \eqref{ex:college.4}, \eqref{ex:college.5}, \eqref{ex:college.6} \,\}$}
to see that its unique world view is the result of adding
$\appointment(\mike)$ to each belief set in~$[\,\eqref{f:sm1}, \eqref{f:sm2}\,]$.

This shows that we can apply the semantics defined above to problems that require to reason about all/some of the answer sets of a program, when this can be done by layers.
In particular, this is interesting for queering databases that may contain disjunctive information as illustrated by Example~\ref{ex:college}.

 \subsection{A guess-define-and-test methodology for conformant planning}
\label{sec:guess-and-check}

The problem of conformant planning consists of finding a sequence of (possibly
concurrent) actions that guarantee the achievement of some goal~\cite{smiwel98a}.
Different to classical planning, the action domain may be nondeterministic, and the initial state may not be completely specified.
A conformant plan is valid if it is guaranteed to be executable and its execution achieves the goal in all possible initial states and all possible effects of the actions.
It is well-known that the problem of finding conformant plans of polynomially-bounded length is \mbox{$\Sigma^P_2$-complete}~\cite{Turner02}.
Quantified Boolean Formulas (QBFs) are one choice for encoding problems in this complexity class.
Indeed, there are QBF encodings for conformant planning, but the logic programming encoding of the problem is much closer to its natural language description and, thus, it is more declarative. This makes the design, understanding and maintenance of the problem solution substantially easier.
It is also well-known that Answer Set Prolog can be used to represent problems on the second level of the polynomial hierarchy~\cite{eitgot95a}.
This may suggest that Answer Set Prolog may be a prime candidate to represent conformant planning problems.
However, tackling problems on the second level of the polynomial hierarchy in Answer Set Prolog usually comes at the cost of using highly sophisticated encodings based on \emph{saturation} that break the intuitive understanding of normal programs.
On the other hand, normal (or head-cycle-free) epistemic programs can also represent problems in the second level of the polynomial hierarchy~\cite{truszczynski11b}, thus constituting an alternative to represent this class of problems.
In particular, semantics satisfying the epistemic splitting property (Property~\ref{property:epistemic.splitting}) provide a natural guess-define-and-test methodology~\cite{martru99a,niemela99a} to represent these problems~\cite{cafafa21a}.
In this methodology, each solution~$\pi$ to the problem at hand corresponds to a set~$S_\pi$ of epistemic literals of the form~$\K a$.
The program is divided into three parts~$\Pi_{\guess}$, $\Pi_{\define}$ and~$\Pi_{\test}$ as follows:
\begin{itemize}
\item The guess part~$\Pi_{\guess}$ generates world views where each of them corresponds to a potential solution of the problem.
Since solutions are encoded using subjective literals of the form~$\K a$, we may assume that each world view~$\wv$ of~$\Pi_{\guess}$ is a singleton satisfying
\begin{gather*}
\K a \vee \K \Not a
	\label{eq:epistemic.choice.property}
\end{gather*}
This means that a solution does not only need to exist, but that a solution needs to be known.
For instance, in the case of conformant planning, this means that the agent needs to know the sequence of actions that it will be performing.

\item The test part~$\Pi_{\test}$ is a set of subjective constraints imposing the conditions to be a solution.
For instance, in the case of conformant planning, this may consist of subjective constraints ensuring that the goal is achieved and that the plan is executable.

\item The define part~$\Pi_{\define}$ is a program, in most cases an objective one, defining auxiliary concepts.
For instance, in the case of conformant planning, this encodes the action domain and the initial state.
\end{itemize}
Let us now illustrate this methodology in more detail.
The use of epistemic logic programs to obtain conformant plans was first advocated by~\citeN{kawabagezh15}.
The semantics used there satisfies neither the epistemic splitting nor the subjective constraint monotonicity properties.
Then, \mbox{\citeN{cafafa21a}} showed that the use of these two properties can greatly simplify the representation.
We follow here this latter approach.
Consider the following variation of the well-known Yale shooting problem~\cite{hanmcd87a} introduced by~\citeN{kawabagezh15}.

\begin{samepage}
\begin{example}
The agent is operating in a domain in which there is a turkey and a gun. The turkey can be alive or not.
The gun may be loaded or not. If the gun is loaded and the trigger is pulled, then the turkey will be dead.
Pulling the trigger will unload the gun. The agent can load the gun, but this action is impossible if the gun is already loaded.
The goal is to kill the turkey.
\end{example}
\end{samepage}

The main difference between this example and the original one introduced by~\citeANP{hanmcd87a} is that we do not know the actual initial state.
The action domain of this example can be represented by objective rules of the form
\begin{align}
\sneg\alive_{i+1} &\leftarrow \trigger_{i},\, \loaded_{i}
	\label{eq:no-alive}
\\
\sneg\loaded_{i+1} &\leftarrow \trigger_{i}
	\label{eq:no-loaded}
\\
\loaded_{i+1} &\leftarrow \load_{i}
	\label{eq:loaded}
\\
\impossible &\leftarrow \load_{i},\, \loaded_{i}
	\label{eq:impossible}
\end{align}
for all~${1 \leq i \leq n}$, where $n$ is a given planning horizon.
Rules~\eqref{eq:no-alive}-\eqref{eq:loaded} describe the effects of the actions,
while rule~\eqref{eq:impossible} captures the fact that \emph{load} cannot occur if the gun is already loaded.
This is an objective program whose representation of the action domain is similar to the one usually used for classical planning in Answer Set Prolog (see for example the paper by~\citeNP{lifschitz02a}).

The initial situation can be represented by the following two disjunctions:
\begin{gather}
\alive_0 \tor \sneg\alive_0
\hspace{2cm}
\loaded_0 \tor \sneg\loaded_0
	\label{eq:initial}
\end{gather}
The define part~$\Pi_\define$ for the conformant planning problem consists of rules~\eqref{eq:no-alive}-\eqref{eq:initial}.
Since this is an objective program, we can use it together with any existing solver for Answer Set Prolog to see how the world evolves if some sequence of actions is executed.
For instance, it is easy to check that program~$\{ \trigger_0,\load_1,\trigger_2 \} \cup \Pi_\define$ has four stable models, that all four contain the literal~$\sneg \alive_3$,
and that none of them contain~$\impossible$.
As a result, we can deduce that~$\pi = \tuple{\trigger_0,\load_1,\trigger_2}$ is a conformant plan, since the goal is achieved in all cases and the sequence of actions is executable.
Using an epistemic logic program, we can encode this final step in the test part.
In this example, program~$\Pi_{\test}$ consists of the following two subjective constraints
\begin{align}
&\leftarrow \Not \K \sneg\alive_n
	\label{eq:goal}
\\
&\leftarrow \M \impossible
	\label{eq:impossible.check}
\end{align}
In particular,~\eqref{eq:goal} states that, in the last situation, the agent must know that the turkey is not alive; while~\eqref{eq:impossible.check} ensures that no impossible action has occurred.
Note that, since~$\Pi_{\define}$ is an objective program, it has a unique world view consisting of all its stable models.
Furthermore, since all those stable models contain~$\sneg \alive_3$, it follows that this unique world view satisfies~$\K \sneg \alive_3$.
Since none of them contain~$\impossible$, we get that it does not satisfy~$\M \impossible$.
Hence, this world view satisfies both~\eqref{eq:goal} and~\eqref{eq:impossible.check}.
As a result, this is the unique world view of the program~$\{ \trigger_0,\load_1,\trigger_2 \} \cup \Pi_\define \cup \Pi_\test$.
This is a direct consequence of subjective constraint monotonicity (Property~\ref{property:constraint.monotonicity}).

Let us now illustrate that this method also allows to show that~$\pi' = \tuple{\load,\trigger}$ is not a conformant plan for this scenario.
In this case, program~$\{ \load_1,\trigger_2 \} \cup \Pi_\define$ also has four stable models and all four contain~$\sneg \alive_3$.
However, two of them also contain~$\impossible$.
As a result, the unique world view of this program does not satisfy constraint~\eqref{eq:impossible.check}, which implies that program~$\{ \load_1,\trigger_2 \} \cup \Pi_\define \cup \Pi_\test$ has no word view at all.
This shows that~$\pi'$ is not a conformant plan because it is not executable in all possible initial situations.
In particular, we cannot $\load$ the gun when it is initially $\loaded$.

Epistemic logic programs can not only be used to check that a sequence of action is a conformant plan, but it can also be used to generate all possible conformant plans.
As usual in Answer Set Prolog, this is achieved by including a choice of the form
\begin{gather}
a_i \tor \Not a_i
	\label{eq:epistemic.choice}
\end{gather}
for all actions at each time step~$1 \leq i < n$.
As mentioned above, for conformant planning, this is not enough because it allows that different actions can be performed for different initial situations.
This is avoided by introducing a rule of the form
\begin{gather}
a_i \leftarrow \M a_i
	\label{eq:epistemic.choice.constraint}
\end{gather}
stating that, if an action $a_i$ occurs in any belief set, it must occur in all of them.
In this sense, $\Pi_{\guess}$ consists of rules of the form~\eqref{eq:epistemic.choice} and~\eqref{eq:epistemic.choice.constraint} for each action and time step.
For instance, in our running example, $\Pi_{\guess}$ consists of the following rules
\begin{gather}
\begin{aligned}
\trigger_i &\tor \Not \trigger_i
\\
\load_i &\tor \Not \load_i
\end{aligned}
\hspace{2cm}
\begin{aligned}
\trigger_i &\leftarrow   \M \trigger_i
\\
\load_i &\leftarrow  \M \load_i
\end{aligned}
	\label{eq:action.guess}
\end{gather}
for $0 \leq i < 3$.
It can be checked that program~$\Pi_\guess \cup \Pi_\define \cup \Pi_\test$
has a unique world view and that this world view satisfies the subjective literals~$\K \trigger_0$, $\K \load_1$ and~$\K \trigger_2$.
That is, the unique world view of this program corresponds to the unique conformant plan for this scenario.

Though we have illustrated the application of the generate-define-and-test methodology for the conformant planning problem, we believe that this methodology can be applied to other problems that fit in the second level of the polynomial hierarchy.
As another example we could consider conformant diagnosis, where we are tasked to find a diagnosis that explains the observations in all possible initial situations.
As with classical planning problems, Answer Set Prolog is well suited to represent diagnostic problems, but not conformant diagnosis ones.
This methodology, not only allow us to represent a conformant diagnosis problem, but also to reuse the existing domain representation used in Answer Set Prolog to represent its non-conformant variation.

 \subsection{Reasoning over attack trees and graphs}
\label{sec:attackgraphs}

The term \emph{attack trees} was coined by Bruce Schneier~\citeyear{Schneier99}, but the concept has most likely existed prior to that. The trees represent chains of attacks that can lead to a goal, the root vertex. Sets of vertices in the tree serve as preconditions (conjunctive or disjunctive) to other vertices, thus forming trees. The idea is to identify conditions that allow for achieving the goal. This notion can be generalized to \emph{attack graphs}, which model collections of attacks and exploits. Usually these graphs are directed acyclic graphs.

In Figure~\ref{fig:attackgraph}, a simplification of a scenario described by \citeN{AlbaneseJN2012} is represented as an attack graph, which in this case forms a tree. The elliptical vertices represent exploits, the other vertices represent conditions or achievements. When exploits have several conditions as predecessors, all of them have to be met to achieve the exploit. Exploits themselves cause new conditions to hold. In turn, for a condition to hold, only one preceding exploit needs to be achieved. Achieving an exploit by leveraging conditions would form an attack. Attacks can lead to conditions that allow for achieving other exploits, this is usually referred to as chains of attacks.

In Figure~\ref{fig:attackgraph} one exploit is ftp\_rhosts(0,2), where an issue in the ftp service is exploited to overwrite the .rhosts file of machine~2. The prerequisite is having a user on machine~0 and ftp access from machine~0 to~2. The exploit causes trust(2,0), which in turn allows for accessing machine~2 from machine~0 via rsh (exploit rsh(0,2)), thus getting user access on machine~2. On the right hand side of the graph having user access to machine~1 and the ability to reach the ssh daemon on machine~2 from machine~1 allows the exploit sshd\_bof(1,2), a buffer overflow exploit in the ssh daemon, which also allows for having user access to machine~2. Having user access to machine~2 allows for the exploit local\_bof(2), which exploits a local buffer overflow issue to gain root access.

\usetikzlibrary{arrows,shapes}

\begin{figure}
  \centering
  
\begin{tikzpicture}[node distance=1.5cm,>=stealth',bend angle=45,auto]

  \tikzstyle{condition}=[rectangle,rounded corners,thick,minimum size=6mm,draw]
  \tikzstyle{exploit}=[ellipse,thick,minimum size=6mm,draw]

  \node [condition] (ftp02) {ftp(0,2)};
  \node [condition] (user0) [right of=ftp02] {user(0)} ;
  \node [exploit] (ftprhosts02) [below of=ftp02,xshift=5mm] {ftp\_rhosts(0,2)};
    \draw [-to,thick] (ftp02) -- (ftprhosts02) {};
    \draw [-to,thick] (user0) -- (ftprhosts02) {};

    \node [condition] (trust02) [below of=ftprhosts02] {trust(2,0)};
      \draw [-to,thick] (ftprhosts02) -- (trust02) {};

      \node [exploit] (rsh02) [below of=trust02] {rsh(0,2)};
      \draw [-to,thick] (trust02) -- (rsh02) {};
      
  \node [condition] (sshd12) [right of=trust02,xshift=10mm] {sshd(1,2)};
  \node [condition] (user0) [right of=sshd12,xshift=3mm] {user(1)} ;
  \node [exploit] (sshdbof12) [below of=sshd12,xshift=5mm] {sshd\_bof(1,2)};
      \draw [-to,thick] (sshd12) -- (sshdbof12) {};
    \draw [-to,thick] (user0) -- (sshdbof12) {};

    \node [condition] (user2) [below of=rsh02,xshift=15mm] {user(2)};
    \draw [-to,thick] (sshdbof12) -- (user2) {};
    \draw [-to,thick] (rsh02) -- (user2) {};

    \node [exploit] (localbof22) [below of=user2] {local\_bof(2)};
    \draw [-to,thick] (user2) -- (localbof22) {};

    \node [condition] (root2) [below of=localbof22] {root(2)};
    \draw [-to,thick] (localbof22) -- (root2) {};
\end{tikzpicture}
\caption{Example attack graph}
\label{fig:attackgraph}
\end{figure}
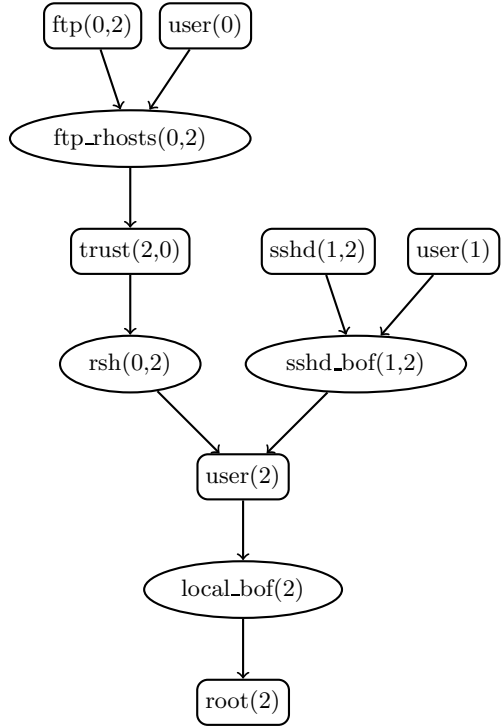

It is clear that graphs of this kind can be modeled by objective rules, in particular, one can create a rule
\begin{gather}
e \leftarrow c_1, \ldots, c_n
	\label{eq:attackgraphs.results}
\end{gather}
for each exploit $e$ and conditions $c_1, \ldots, c_n$ that point to it. For conditions caused by an exploit, we create rules
\begin{gather}
c \leftarrow e
	\label{eq:attackgraphs.exploits}
\end{gather}
for each condition $c$ caused by exploit $e$. If some exploits are already known to hold or not to hold, they can be affirmed as facts. For the example in Figure~\ref{fig:attackgraph}, this would lead to the following rules:
\begin{gather*}
  \begin{aligned}
    \mathit{ftp\_rhosts}(0,2) & \leftarrow \mathit{ftp}(0,2), \mathit{user}(0)\\
    \mathit{trust}(2,0) & \leftarrow \mathit{ftp\_rhosts}(0,2)\\
    \mathit{rsh}(0,2) & \leftarrow \mathit{trust}(2,0)\\
    \mathit{sshd\_bof}(1,2) & \leftarrow \mathit{sshd}(1,2), \mathit{user}(1)\\
    \mathit{user}(2) & \leftarrow \mathit{rsh}(0,2)\\
    \mathit{user}(2) & \leftarrow \mathit{sshd\_bof}(1,2)\\
    \mathit{local\_bof}(2) & \leftarrow \mathit{user}(2)\\
    \mathit{root}(2) & \leftarrow \mathit{local\_bof}(2)
  \end{aligned}
\end{gather*}

One advantage of Epistemic Logic Programs over the graph representation is the ability to abstract from specific machines and thus arrive at more compact representations. For example, the following rule represents that having ftp access from a machine X to another machine Y and having a user on X gives rise to the  ftp\_rhosts exploit from X to Y:
\begin{gather*}
  \begin{aligned}
      \mathit{ftp\_rhosts}(X,Y) & \leftarrow \mathit{ftp}(X,Y), \mathit{user}(X)
  \end{aligned}
\end{gather*}

In order to use the attack graph to establish that the exploit $\mathit{root}(2)$ can not be reached under any circumstances, one can add a choice of the form
\begin{gather}
c \tor \sneg c
	\label{eq:attackgraphs.choice}
\end{gather}
for each condition $c$ with no incoming arc, and the constraint
\begin{gather}
\leftarrow \M e
	\label{eq:attackgraphs.secure_constraint}
\end{gather}
for each exploit $e$ to be checked.

In the example, this gives rise to
\begin{gather*}
  \begin{aligned}
    \mathit{ftp}(0,2) \tor \sneg \mathit{ftp}(0,2)\\
    \mathit{user}(0) \tor \sneg \mathit{user}(0)\\
    \mathit{sshd}(1,2) \tor \sneg \mathit{sshd}(1,2)\\
    \mathit{user}(1) \tor \sneg \mathit{user}(1)\\
    \leftarrow \M \mathit{root}(2)
  \end{aligned}
\end{gather*}

This modeling allows for flexibility in further reasoning. As an example, we consider the application of hardening measures, as they are called by~\citeN{AlbaneseJN2012}, in order to close vulnerabilities. In the running example, hardening measures could be to close ftp and/or ssh access to the network. One could add ``epistemic guesses'' for each of hardening measures and rules that describe their consequences as follows:

\begin{gather*}
  \begin{aligned}
  \mathit{close\_ftp} & \leftarrow \Not \K \sneg \mathit{close\_ftp}\\
  \sneg \mathit{close\_ftp} & \leftarrow \Not \K \mathit{close\_ftp}\\
  \mathit{close\_sshd} & \leftarrow \Not \K \sneg \mathit{close\_sshd}\\
  \sneg \mathit{close\_sshd} & \leftarrow \Not \K \mathit{close\_sshd}\\
  \sneg \mathit{ftp(0,2)} & \leftarrow \mathit{close\_ftp}\\
  \sneg \mathit{sshd(1,2)} & \leftarrow \mathit{close\_sshd}
  \end{aligned}
\end{gather*}

Each subset of available hardening measures thus gives rise to a potential world view, but only those that guarantee that the exploit is impossible to achieve will be world views, due to~\eqref{eq:attackgraphs.secure_constraint}. 
Note that this guarantee is provided by the subjective constraint monotonicity property (Property~\ref{property:constraint.monotonicity}).
This property is satisfied by the semantics we discuss in Section~\ref{sec:original} and some other semantics we will see in the next section.
However, there are also semantics that do satisfy this property and for which this claim may not hold.

In the example, there is one potential world view in which $\sneg \mathit{close\_ftp}$ and $\sneg \mathit{close\_sshd}$ hold, but clearly it contains several answer sets (containing both $\mathit{ftp}(0,2)$ and $\mathit{user}(0)$ or both $\mathit{sshd}(1,2)$ and $\mathit{user}(1)$) that contain $\mathit{root}(2)$ and therefore violate the constraint. Similarly, the world view in which  $\mathit{close\_ftp}$ and $\sneg \mathit{close\_sshd}$ hold, will contain an answer set that contains $\mathit{sshd}(1,2)$ and $\mathit{user}(1)$ and therefore also $\mathit{root}(2)$, while the world view in which  $\sneg \mathit{close\_ftp}$ and $\mathit{close\_sshd}$ hold, will contain an answer set that contains $\mathit{ftp}(0,2)$ and $\mathit{user}(0)$ and therefore also $\mathit{root}(2)$. Only the world view in which  $\mathit{close\_ftp}$ and $\mathit{close\_sshd}$ hold can guarantee that $\mathit{root}(2)$ is false in each of its answer sets. Indeed, in this example the only hardening that avoids the exploit $\mathit{root}(2)$ is to close both ftp and ssh access.
  \section{The search for self-support-free world views}
\label{sec:semantics}

As mentioned in Section~\ref{sec:inception},
early formalizations of epistemic specifications contained unsupported beliefs.
In this section, we review the major approaches that have addressed this issue.
In Subsection~\ref{sec:foundedness}, we also review a new property called \emph{foundedness} that aims to capture the essence of self-supported-free world views in a formal way.
We use this property, together with the subjective constraint monotonicity and epistemic splitting properties defined earlier, to provide a formal comparison between different approaches.
We also go deeper in this comparison by providing some translations between approaches and identifying some agreement subclasses. 

With the exception of the work by~\citeN{sheeit16},
all the existing work addressing this issue focused on ground theories.
Thus, in the rest of this section, we restrict ourselves to ground theories.
Another interesting point to mention are the languages used by different approaches.
As mentioned earlier,
the original language of
epistemic specifications contained two modalities~$\K$ and~$\M$.
\begin{table}[h]
\centering
\begin{tabular*}{12.25cm}{ @{\hskip1cm} c @{\hskip2.25cm}  @{\hskip.75cm} c c c  }
\hline\hline
primitive 			& &defined operators
\\\hline
   			& $\K\fF$ & $\M\fF$ & $\eNot \fF$
\\\hline
$\K$		& - & $\Not\K\Not\fF$ & $\Not\K\fF$
\\\hline
$\M$		& $\Not\M\Not\fF$ & - & $\M\Not\fF$
\\\hline
$\eNot$		& $\Not\eNot\fF$ & $\eNot\Not\fF$ & -
\\\hline\hline
\end{tabular*}
	\caption{Interdefinability of epistemic operators.
	If the rewriting produces a formula of the form~$\Not\Not\Not \varphi$,
	it is replaced by~$\Not \varphi$.
	For instance, rewriting~$\Not\Not \M \varphi$ produces~$\Not\Not\Not \K \Not\varphi$ and, thus, we get $\Not \K \Not \varphi$.}
	\label{table:interdefinition.opeartors}
\end{table}
Interestingly,
according to the semantics given so far,
these two modalities are interdefinable as shown it Table~\ref{table:interdefinition.opeartors}.
This interdefinability holds for all approaches we review below with two exceptions:
\citeN{faheir15a} presented a semantics where the above equivalences do not hold, and~\citeN{cafafa19a} leaves the discussion about the modal operator~$\M$ for future work.

\citeN{sheeit17a} introduced a third modal operator~$\eNot$
where~$\eNot \fF$ can be read as ``there is no evidence proving that~$\fF$ is true.''
Interestingly,
this third modal operator is also interdefinable with the other two.
In view of these facts,
in the following we focus on reviewing the different semantics for the~$\K$ operator
and assume that, unless stated otherwise,~$\M$ and~$\eNot$ are treated as abbreviations following Table~\ref{table:interdefinition.opeartors}.

 \subsection{Gelfond 2011}\label{sec:gelfond2011}

\citeN{gelfond11a} was the first to discuss the existence of unintended world views in the early works on epistemic specifications and to propose an alternative semantics.
As mentioned in the introduction, an example of these unintended world views is the existence of the unsupported belief~$p$ in one of the world views of the program consisting of rule
\begin{gather}
p \leftarrow \K p.
	\tag{\ref{es:self-supported}}
\end{gather}
We can check that, according to Definition~\ref{def:G94-world.view}, this program has two world views:
$[\{ \, \}]$ and $[\{p\}]$.
For the former, note that $[\{ \, \}] \not\models \K p$.
Then, the G94-reduct of~$p \leftarrow \K p$ with respect to~$[\{ \, \}]$ is the tautological rule
\begin{gather}
p \leftarrow \bot.
	\label{eq:self-supported.G94-reduct.empty}
\end{gather}
It is easy to see that this objective program is equivalent to the empty program and, thus, it has the empty set as its unique stable model.
As a result, we obtain that $[\{ \, \}]$ is indeed a world view of program~$\{ p \leftarrow \K p \}$.
On the other hand, we can see that $[\{  p \}] \models \K p$.
As a result, the G94-reduct of~$p \leftarrow \K p$ with respect to~$[\{  p  \}]$ is
\begin{gather}
p \leftarrow \top.
	\label{eq:self-supported.G94-reduct.p}
\end{gather}
The unique stable model of this objective program is~$\{p\}$ and, therefore, $[\{  p  \}]$ is also a G94-world view of~$\{ p \leftarrow \K p \}$.

Motivated by this issue, \citeN{gelfond11a} proposed the following variation of the reduct.
The definition of G11-world views is exactly as the definition of G94-world views, but it uses this new reduct instead of the G94-reduct.

\begin{definition}[G11-reduct and world views]
	\label{def:g11.reduct}
Given a logic program~$\Pi$, its G11-reduct with respect to a non-empty set of interpretations~$\wv$ 
is the program obtained by:
\begin{enumerate}
\item replacing by $\bot$ every subjective literal~$L$ such that $\wv \not\models L$;
\item removing all other occurrences of subjective literals in the scope of default negation;
\item replacing all other occurrences of subjective literals of the form $\K l$ by $l$.
\end{enumerate}
An epistemic interpretation~$\wv$ is a G11-world view of $\Pi$ iff $\wv$ is the set of all stable models of the G11-reduct of $\Pi$ with respect to~$\wv$.
\end{definition}

Definition~\ref{def:g11.reduct} was an attempt to find a formalization of the Rationality Principle.
The main technical tool used for this purpose was this new reduct.
Unlike other existing ASP reducts which normally remove a program's rule or some extended literal from the rule's body, the new reduct allowed replacement of an epistemic literal~$\K l$ by its corresponding objective literal~$l$.
The intention was to ensure that the rule allows the inclusion of the head in a particular belief set only if it already contains the objective literals corresponding to all the epistemic literals in its body.
It worked for rules like $p \leftarrow \K p$ and other simple examples, but failed to
completely eliminate unintended beliefs (see an example below).

Continuing with our running example, we can see now that the G11-reduct of~$p \leftarrow \K p$ with respect to the epistemic interpretation~$[\{ \, \}]$ is the same as its \mbox{G94-reduct}.
Therefore, $[\{ \, \}]$ is also a G11-world view of~$\{ p \leftarrow \K p\}$.
In contrast, the G11-reduct of~$p \leftarrow \K p$ with respect to~$[\{ p \}]$
is the tautology
\begin{gather}
p \leftarrow p.
	\label{eq:self-supported.G11-reduct.p}
\end{gather}
The unique stable model of this program is the empty set and, thus, $[\{ p \}]$ is not a G11-world view of~$\{ p \leftarrow \K p\}$.

It worth noting that, if all occurrences of epistemic literals are in the scope of negation, this semantics coincide with the G94-semantics introduced above.
On the other hand, if no epistemic literal occurs in the scope of negation, this semantics coincide with the K15-semantics introduced in the next section.

Despite the success of this semantics in removing the unsupported belief~$p$ in the simple example given above, it still presents unsupported beliefs in more complex examples.
Take for instance the following program used by~\citeN{cafafa19a} to illustrate this fact:
\begin{gather}
p \tor q
\hspace{2cm}
p \leftarrow \K q
\hspace{2cm}
q \leftarrow \K p.
\label{theory:larger.set.of.worlds}
\end{gather}
We can see that this program has two G94-world views, $\sset{ \set{p}, \set{q} }$ and
$\sset{ \set{p , q} }$.
Here, the belief~${p \wedge q}$ in the second world view is unsupported.
Note that the first rule does not support~${p \wedge q}$ and the other two rules only can support this fact if $p$ or~$q$ were supported in all belief sets, which is not the case.
Still, this second G94-world view is also a \mbox{G11-world} view.
To see why, note that the G11-reduct of program~$\{\eqref{theory:larger.set.of.worlds}\}$ with respect to $\sset{ \set{p , q} }$
is the objective program:
\begin{gather*}
p \tor q
\hspace{2cm}
p \leftarrow q
\hspace{2cm}
q \leftarrow p
\end{gather*}
which has the unique stable model $\set{p,q}$.
In fact, we can see below that this example (or a slight variations of it) provides a major challenge to most existing approaches.

It is worth mentioning that this semantics satisfies subjective constraint monotonicity (Property~\ref{property:constraint.monotonicity}; see the paper by~\citeNP{fandinno19a}), although it does not satisfy the epistemic splitting property (Property~\ref{property:epistemic.splitting}).
To see that this semantics does not satisfy epistemic splitting,
take the following program from the paper by~\citeN{cafafa21a}:
\begin{gather}
p \tor q 
\hspace{2cm}
s \leftarrow  \K p
\hspace{2cm}
\leftarrow \Not s
	\label{prg:g11.no.epistemic.splitting}
\end{gather}
This program has no G94-world view.
Note that~$U=\{p,q\}$ is a splitting set for this program that divides it in a bottom part~$\{p \tor q\}$ and a top part~$\{ s \leftarrow \K p,\, \leftarrow \Not s\}$.
It is easy to see that the unique world view of the bottom part is~$[\{p\},\,\{q\}]$
and, simplifying the top part with respect to this world view, we obtain
the unsatisfiable program
\begin{gather*}
c \leftarrow  \bot
\hspace{2cm}
\leftarrow \Not c
\end{gather*}
Since the G94-semantics satisfies the epistemic splitting property, this immediately implies that this program has no G94-world view.
Similarly, if the G11-semantics would satisfy the epistemic splitting property, we would expect that this program had no \mbox{G11-world} view either.
However, this program does have the G11-world view~$[\{p,s \}]$.
To see this fact, note that the G11-reduct of the program containing the rules~\eqref{prg:g11.no.epistemic.splitting} with respect to~$[\{p,s \}]$
is
\begin{gather*}
p \tor q
\hspace{2cm}
s \leftarrow  p
\hspace{2cm}
\leftarrow \Not s.
\end{gather*}
The unique stable model of this objective program is~$\set{p,s}$.
This program also illustrates another example of unsupported beliefs:
here neither of the beliefs~$p$ and~$s$ are supported.
This can be easily seen by removing the constraint~$\leftarrow \Not s$.
The resulting program
\begin{gather}
p \tor q 
\hspace{2cm}
s \leftarrow  \K p
	\label{prg:g11.no.epistemic.splitting.no.constraint}
\end{gather}
has a unique world view~$[\{p\},\{q\}]$ in all semantics discussed in this paper.
Here, neither~$p$ nor~$s$ are believed, and adding the constraint~$\leftarrow \Not s$ does not provide any reason to believe either of them.
 \subsection{Kahl, Watson, Balai, Gelfond \&Zhang 2015}\label{sec:kahl2015}

\citeN{kawabagezh15} revised the semantics introduced by~\citeN{gelfond11a} in order to avoid the presence of multiple world views due to recursion through the operator~$\M$.
As an example of this issue consider the program consisting of rule
\begin{gather}
p \leftarrow \M p.
	\label{eq:self-supported.M}
\end{gather}
With respect to the G94- and the \mbox{G11-semantics}, this program has two worlds views, namely $[\{\,\}]$ and~$[\{ p \} ]$.
\citeN{kawabagezh15} argue that this program should have a unique world view and that it should be~$[\{ p \} ]$.
This argument is based on the following observation.
\begin{quote}
	It was observed by looking at the definitions for satisfiability that a rational agent should find it easier
	to accept certain extended literals over others. This is clear when we look at the fact that, e.g. given
	a belief interpretation\footnote{\citeN{kawabagezh15} use the term ``pointed ES structure'' instead of ``belief interpretation.'' We made the replacement here to keep the coherence with the rest of the text.}
$\tuple{\wv,I}$, in order to establish $\K l$, it must be demonstrated that $l$ belongs to all
	belief sets in $\wv$.
To establish $l$, it must be demonstrated that $l$ belongs to a particular belief set in~$\wv$, namely~$I$.
But to establish $\M l$, it is sufficient to demonstrate that~$l$ belongs to some belief set in~$\wv$.
\end{quote}
As a result, a preference order among literals~$\K l$, $l$ and~$\M l$ is established, where~$\K l$ is the hardest to accept (or it is the one that requires the highest degree of conviction) and~$\M l$ is the easiest to accept.
Taken into account this preference order, one can deduce that the fact that accepting~$l$ requires it to be self-support-free does not imply that accepting~$\M l$ should require it to be self-support-free, too.
Unfortunately, this observation does not imply the contrary either; and other authors, like~\citeN{sufahe20}, have opted to require that~$\M l$ should be self-support-free.
According to~\citeN{sufahe20}, the unique world view of this program should be~$[\{\,\}]$.
The intuition of recursion through the~$\M$ operator is subject of open debate and one could even develop a semantics with two different~$\M$-like modal operators, where one of them requires to be self-support-free and the other does not.

Focusing on the semantics introduced by~\citeN{kawabagezh15}, recall that following our convention (see Table~\ref{table:interdefinition.opeartors}), rule~\eqref{eq:self-supported.M} is an abbreviation for rule
\begin{gather}
p \leftarrow \Not \K \Not p.
	\label{eq:self-supported.M.translated}
\end{gather}
Note that $[\{ \, \} ] \models \K \Not p$.
Therefore, the G94-reduct of~\eqref{eq:self-supported.M.translated} with respect to~$[\{ \, \} ]$ is
\begin{gather*}
p \leftarrow \Not \top,
\end{gather*}
while its G11-reduct is
\begin{gather*}
p \leftarrow \bot.
\end{gather*}
It is easy to see that these two rules are equivalent and their unique stable model is the empty set.
As a result, $[\{ \, \} ]$ is both a G94- and a G11-world view.

Motivated by this issue, \citeN{kawabagezh15} proposed a new variation of the reduct.
The definition of K15-world views is exactly as the definition of G94- and \mbox{G11-world} views, but using this new reduct instead.
We present here the definition introduced by~\citeN[Appendix~C]{kawabagezh15}.

\begin{definition}[K15-reduct and world view]
\label{eq:K15-reduct}
The \emph{K15-reduct} of a ground program~$\Pi$ 
with respect to an epistemic interpretation~$\wv$ is obtained by replacing each maximal subformula of the form~$\K l$ by $l$, if $\wv \models \K l$, or by~$\bot$, otherwise.\footnote{If replacing~$\K l$ by $l$ results in more than two nested default negations we simplify it using the following rewriting rule recursively:
$\Not\Not\Not \fF \mapsto \Not \fF$.}

An epistemic interpretation~$\wv$ is a K15-world view of $\Pi$ iff $\wv$ is the set of all stable models of the K15-reduct of $\Pi$ with respect to~$\wv$.
\end{definition}

Continuing with our running example, we can now see that the K15-reduct of~\eqref{eq:self-supported.M.translated} with respect to $[\{ \, \} ]$ is
\begin{gather*}
p \leftarrow \Not \Not p,
\end{gather*}
and the resulting program has two stable models: $\{\, \}$ and~$\{p\}$.
Therefore, $[\{ \, \} ]$ is not a K15-world view.
In a similar way, we can check that $[\{ p \} ]$ is indeed a K15-world view.
Note that the K15-reduct with respect to this epistemic interpretation is
\begin{gather*}
p \leftarrow \Not \bot.
\end{gather*}

It is worth mentioning that for programs, in which the epistemic operator does not occur under the scope of default negation, this semantics coincides with G11.
As a result, the arguments stated in Section~\ref{sec:gelfond2011} for programs~\eqref{theory:larger.set.of.worlds} and~\eqref{prg:g11.no.epistemic.splitting} also apply to this semantics.
This implies that this semantics also manifests unsupported beliefs and that it does not satisfy epistemic splitting.
Besides, as observed by~\citeN{leckah18b}, this semantics does not satisfy subjective constraint monotonicity (Property~\ref{property:constraint.monotonicity}), while G11 does.
The following program taken from the paper by~\citeN{leckah18b} illustrates this fact:
\begin{gather}
p \tor q
\hspace{2cm}
\leftarrow  \Not \K p.
	\label{eq:leckah18b}
\end{gather}
This program has a unique K15-world view $\cset{ \set{ p } }$.
Note however that the program $\{ p \tor q \}$ is objective and has two stable models~$\set{p}$ and~$\set{q}$.
Thus, it has the unique world view~$\cset{ \set{ p }, \set{ q } }$,
which does not satisfy the subjective constraint~$\leftarrow \Not  \K p$.
We can see that adding this subjective constraint makes $\cset{ \set{ p } }$ a world view, which contradicts this property.

An interesting observation about this semantics is that it can be considered as the \emph{reflexive} counterpart of G94 in a sense similar to the relation between Moore's autoepistemic logic and reflexive autoepistemic logic~\cite{schwarz91a}.
In fact, we can use the embedding from reflexive autoepistemic logic into Moore's autoepistemic logic~\cite[page~304]{martru93a} to illustrate this fact.
This embedding~$(\,\cdot\,)^B$ is defined recursively as follows:
\begin{enumerate}
\item $\fF^B = \fF$ if $\fF$ is an atom or~$F = \bot$;

\item $(\sneg\fF)^B = \sneg(\fF^B)$

\item $(\fF \otimes \fG)^B = \fF^B \otimes \fG^B$ for $\otimes \in \{\wedge, \vee, \to \}$; 

\item $(Qx\, \fF(x))^B = Qx\, \fF(x)^B$ for $Q \in \{\forall, \exists \}$; and

\item $(\K \fF)^B = \fF^B \wedge \K \fF^B$.
\end{enumerate}
For a theory~$\Gamma$, the embedding is defined as~$\Gamma^B = \{ \fF^B \mid \fF \in \Gamma \}$.

\begin{proposition}
\label{prop:reflexive.embedding}
The K15-word views of any program~$\Pi$ coincide with the G94-world views of $\Pi^B$.
\end{proposition}

\begin{proof*}
It is enough to show that the K15-reduct of~$\Gamma$ is equivalent to the G94-reduct of $\Gamma^B$ for any epistemic interpretation~$\wv$.
Pick any maximal subformula of the form~$\K l$.
We proceed by cases.
\begin{itemize}
\item If $\wv \models \K l$, then the K15-reduct replaces it by~$l$.
On the other hand, the \mbox{G94-reduct} of $(\K l)^B = l^B \wedge \K l^B = l \wedge \K l$ is $l \wedge \top$.
Clearly these two formulas are equivalent.
Note also that~$l^B = l$ because~$l$ is an objective literal.

\item Otherwise, the K15-reduct replaces $\K l$ by~$\bot$
while the \mbox{G94-reduct} of $(\K l)^B = l \wedge \K l$ is $l \wedge \bot$, which are also equivalent.\hspace*{1em}\hbox{\proofbox}
\end{itemize}
\end{proof*}
Proposition~\ref{prop:reflexive.embedding} gives us a straightforward way to extend the K15-semantics from ground logic programs to arbitrary theories:
we can define the K15-world views of any theory~$\Gamma$ as the G94-world views of~$\Gamma^B$.

A converse embedding~$(\,\cdot\,)^K$ from the G94- into the K15-semantics is also possible.
This embedding is defined recursively as follows:
\begin{enumerate}
\item $\fF^K = \fF$ if $\fF$ is an atom or~$F = \bot$;

\item $(\sneg\fF)^K = \sneg(\fF^K)$

\item $(\fF \otimes \fG)^K = \fF^K \otimes \fG^K$ for $\otimes \in \{\wedge, \vee, \to \}$; 

\item $(Qx\, \fF(x))^K = Qx\, \fF(x)^K$ for $Q \in \{\forall, \exists \}$; and
\item $(\K \fF)^K = \M\K \fF^K$.
\end{enumerate}
Note that~$\cdot^K$ differs from $\cdot^B$ only in the last condition.

\begin{proposition}
\label{prop:reflexive.embedding2}
The K15-word views of any theory~$\Gamma$ coincide with the G94-world views of $(\Gamma^K)^B$.
\end{proposition}

\begin{proof*}
It is enough to show that the G94-reduct of~$\Gamma$ is equivalent to G94-reduct of $(\Gamma^K)^B$ for any epistemic interpretation~$\wv$.
Pick any maximal subformula of the form~$\K F$.
We proceed by cases.
\begin{itemize}
\item If $\wv \models \K F$, then the G94-reduct replaces it by~$\top$.
On the other hand, we have
\begin{align*}
((\K \fF)^K)^B
	&= (\M\K \fF^K)^B 
\\
	&= (\Not \K \Not \K \fF^K)^B
\\
	&= \Not (\K \Not \K \fF^K)^B
\\
	&= \Not (\Not \K (\fF^K)^B \wedge \K \Not \K (\fF^K)^B)
\\
	&\equiv \Not \Not \K (\fF^K)^B \vee \Not \K \Not \K (\fF^K)^B)
\\
	&\equiv \K (\fF^K)^B \vee \Not \K \Not \K (\fF^K)^B)
\end{align*}
Furthermore, it can be checked by induction that $\wv\models\K \fF$ iff $\wv\models\K(\fF^K)^B$.
Hence, the G94 reduct of $((\K \fF)^K)^B$ is~$\top \vee \Not \bot \equiv \top$
and, thus, equivalent to the G94-reduct of~$\K \fF$.

\item Otherwise, the G94-reduct replaces $\K F$ by~$\bot$
while the G94 reduct of $((\K \fF)^K)^B$ is~$\bot \vee \Not \top \equiv \bot$
and, thus, equivalent to the G94-reduct of~$\K \fF$.\hspace*{1em}\hbox{\proofbox}
\end{itemize}
\end{proof*}

\begin{corollary}
\label{cor:reflexive.embedding2}
The G94-world views of any program~$\Pi$ coincide with the K15-world views of $\Pi^K$.
\end{corollary}

\begin{proof}
From Proposition~\ref{prop:reflexive.embedding}, the K15-world views of $\Pi^K$ coincide with the \mbox{G94-world} views of~$(\Pi^K)^B$ which, from Proposition~\ref{prop:reflexive.embedding2} coincide with with the G94-world views of~$\Pi$.\end{proof}

Propositions~\ref{prop:reflexive.embedding} and Corollary~\ref{cor:reflexive.embedding2} provide a formal correspondence between the G94 and the K15 semantics.
Furthermore, they also show that any tool to compute the world views of one semantics can be used to compute the world views of the other, with the minimum effort of applying this translation.
Note however that, in general, even if~$\Pi$ is a program, neither~$\Pi^B$ nor $\Pi^K$ are necessarily programs.
For instance, $\eqref{eq:self-supported.M.translated}^B$ is
\begin{gather}
p \leftarrow \Not\, (\Not p \wedge \K \Not p)
	\label{eq:self-supported.M.^B}
\end{gather}
which syntactically is not a valid rule.
However, this formula can be transformed into a set of rules that is equivalent modulo the original signature (similar to the~\citeANP{tseitin68a} transformation from~\citeyearNP{tseitin68a}).
In this case, we may rewrite~\eqref{eq:self-supported.M.^B} as
\begin{align}
p &\leftarrow \Not aux
	\label{eq:self-supported.M.^B.1}
\\
aux &\leftarrow \, \Not p \wedge \K \Not p
	\label{eq:self-supported.M.^B.2}
\end{align}
The unique G94-world view of this program is~$[\{ p \} ]$, which coincides with the unique K15-world view of~$\eqref{eq:self-supported.M.translated}$.
This translation was used (without proof) by the solver~\texttt{eclingo} to compute the G94- and K15-world views using the same tool~\cite{cafagarosc20aURL}.
 \subsection{Shen \& Eiter 2016}\label{sec:shen2017}

\citeANP{sheeit16}~(\citeyearNP{sheeit16,sheeit17a}) introduced the operator $\eNot$ with the intention to evaluate $\eNot l$ to true if $l$ is false in at least one belief set of a world view, which intuitively corresponds to $\Not \K l$ or $\M \Not l$, as noted earlier.
Their key idea was to treat $\eNot$ in a similar way as $\Not$ (default negation) and assume the truth of $\eNot l$ whenever possible.
This notion was named \emph{knowledge minimization with epistemic negation}.

According to \citeN{sheeit16}, \emph{candidate world views} are defined as a first step, followed by a minimization criterion (by maximizing negative knowledge). The following equivalent definition, which relates S16 to K15-world views, has been proposed by~\citeN{kaleso16} and by~\citeN{solekale17}.

\begin{definition}[S16-world views]
Let~$\Pi$ be a logic program~$\Pi$ and $E_\Pi$ be the set of epistemic literals that contains $\Not \K l$ for every epistemic literal of the form $\K l$ that occurs in~$\Pi$.
Let $\Phi_\wv \eqdef \setm{ L \in E_\Pi }{\wv \models L}$ be the subset of $E_\Pi$ satisfied by an epistemic interpretation~$\wv$.
Then, epistemic interpretation~$\wv$ is a S16-world view iff it is a K15-world view and there is no other K15-world view~$\wv'$ such that $\Phi_{\wv'} \supset \Phi_{\wv}$.
\end{definition}

Given this definition, it is clear that the single K15-world view $\cset{ \set{ p } }$ of the program~$\eqref{eq:leckah18b}$ reported earlier is also the single S16-world view. This shows that S16 does not satisfy the subjective constraint monotonicity nor the epistemic splitting properties either.

The following example from the paper by~\citeN{sheeit16} illustrate differences between semantics K15 and S16.
\begin{gather}
p \leftarrow \M q, \Not q. \hspace{1cm} q \leftarrow \M p, \Not p.
	\label{eq:sheeit16-ex1}
\end{gather}
Expanding $\M$ according to Table~\ref{table:interdefinition.opeartors} yields:
\begin{gather}
p \leftarrow \Not \K \Not q, \Not q. \hspace{1cm} q \leftarrow \Not \K \Not p, \Not p.
	\label{eq:sheeit16-ex1.translated}
\end{gather}
There are two K15-world views, $\wv_1 = [\{p\},\{q\}]$ and $\wv_2 = [\{\,\}]$. As $\wv_1 \not\models \K \Not q$ and $\wv_1 \not\models \K \Not p$, the K15-reduct of $\eqref{eq:sheeit16-ex1.translated}$ with respect to $\wv_1$ is 
\begin{gather}
p \leftarrow \Not \bot, \Not q. \hspace{1cm} q \leftarrow \Not \bot, \Not p.
	\label{eq:sheeit16-ex1.translated^pq}
\end{gather}
the stable models of which are $\{p\}$ and $\{q\}$, so $\wv_1$ is a K15-world view. Next observe that $\wv_2 \models \K \Not q$ and $\wv_2 \models \K \Not p$, so the K15-reduct of $\eqref{eq:sheeit16-ex1.translated}$ with respect to $\wv_2$ is
\begin{gather}
p \leftarrow \Not \Not q, \Not q. \hspace{1cm} q \leftarrow \Not \Not p, \Not p.
	\label{eq:sheeit16-ex1.translated^empty}
\end{gather}
which has the single stable model $\{\,\}$, so $\wv_2$ is also a K15-world view.

Finally, observe that $\Phi_{\wv_1} = \{\Not \K \Not p, \Not \K \Not q\}$ and $\Phi_{\wv_2} = \{\,\}$, so $\wv_1$ is an \mbox{S16-world} view, but $\wv_2$ is not, because $\Phi_{\wv_1} \supset \Phi_{\wv_2}$.
 \subsection{Fari\~{n}as del Cerro, Herzig \& Iraz Su 2015}\label{sec:farinas2015}

\citeN{faheir15a} tackle the issue of self-supported beliefs by introducing a modal extension of equilibrium logic rather than a variation of the reduct approaches.
Defining a modal extension follows the common practice in intuitionistic modal logics~\cite{fisher77,farinas83,simpson1994,biedep00a}.
In this case, equilibrium logic is extended with modal logic~S5.
As usual, this modal extension properly distinguishes between the modal operators~$\K$ and~$\M$ that, to date, are not known to be interdefinable.
We mostly follow here the revised presentation by~\citeN{irazu20}.
Note that the operator~$\M$ is written as~$\hat{\K}$ there.

Formally,
an \emph{F15-interpretation} is a pair $\tuple{\wv,h}$ where $\wv$ is an epistemic interpretation and $h: \wv \longrightarrow 2^\At$ is a function mapping each interpretation~$T \in \wv$ to some subset of atoms such that $h(T) \subseteq T$.
Satisfaction of formulas with respect to F15-interpretations is defined in a similar way as with respect to belief interpretations.
Satisfaction of a formula~$\fF$ with respect to an F15-interpretation $\tuple{\wv,h}$ and a propositional interpretation~$I$
is recursively defined as follows:
\begin{enumerate}
\item $\tuple{\wv,h},I \models a$ iff $a \in I$, for any atom $a \in \At$;
\item $\tuple{\wv,h},I \models \fG_1 \wedge \fG_2$ iff $\tuple{\wv,h},I \models \fG_1$ and $\tuple{\wv,h},I \models \fG_2$;
\item $\tuple{\wv,h},I \models \fG_1 \vee \fG_2$ iff $\tuple{\wv,h},I \models \fG_1$ or $\tuple{\wv,h},I \models \fG_2$;
\item $\tuple{\wv,h},I \models \fG_1 \to \fG_2$ iff $\tuple{\wv,h'},I \not\models \fG_1$ or $\tuple{\wv,h'},I \models \fG_2$ for both $h' \in \{h,\mathit{id}\}$;
\item $\tuple{\wv,h},I \models \K \fG$ iff $\tuple{\wv,h,J} \models \fG$ for all $J \in \wv$; and
\item $\tuple{\wv,h},I \models \M \fG$ iff $\tuple{\wv,h,J} \models \fG$ for some $J \in \wv$.
\end{enumerate}
where $id : \wv \longrightarrow 2^\At$ is the identity function, that is, $id(T) = T$ for every $T \in \wv$.
We say that 
an F15-interpretation $\tuple{\wv,h}$
\emph{satisfies} a formula~$\fF$ when
$\tuple{\wv,h},I \models \fF$ for all $I \in \wv$.
In this case~$\tuple{\wv,h}$ is also called an \emph{F15-model} of~$\fF$.
We say that $\tuple{\wv,h}$ is an \emph{F15-model} of a theory~$\Gamma$, written $\tuple{\wv,h}\models \Gamma$, if it is an F15-model of all its formulas~$\fF \in \Gamma$.

Given an epistemic interpretation $\wv$ and an F15-interpretation~${\cI = \tuple{\wv',h}}$,
we write ${\cI \preceq \wv}$ if ${\wv = \wv'}$ and~$h(I) \subseteq I$ for all~$I \in \wv$.
We write ${\cI \prec \wv}$ if ${\cI \preceq \wv}$ and $h \neq id$.
Then, equilibrium models are defined as follows:

\begin{definition}[F15-equilibrium model]
An epistemic interpretation~$\wv$
is called an \emph{F15-equilibrium model} of a theory~$\Gamma$ 
if it is a model of~$\Gamma$
and there is no F15-model~$\cI$ of~$\Gamma$ with~${\cI \prec \wv}$.
\end{definition}

The F15-world views are obtained from a selection among equilibrium F15-models.
For defining that selection, we need to introduce the following terminology.
A function $h : \wv \longrightarrow 2^{\At}$ is said to be \emph{total on} a set $\wx \subseteq \wv$ iff $h(I) = I$ for every $I \in \wx$.

\begin{definition}
Given a theory~$\Gamma$, an epistemic interpretation~$\wv$ and a subset $\wx \subseteq \wv$ of it,
we write $\wv,\wx \models^* \Gamma$ if the following two conditions are satisfied:
\begin{enumerate}
\item $\tuple{\wv,id},I \models \Gamma$ for all $I \in \wx$, and
\item every~${h \neq  id}$ that is total on~${\wv\setminus\wx}$  satifies $\tuple{\wv,h},I \not\models \Gamma$ for some ${I \in \wx}$.
\end{enumerate}
For any epistemic interpretations $\wv$ and $\wv'$
we write
$\wv \leq_\Gamma \wv'$ if
\begin{gather*}
\wv \cup \set{ I}, \wv \models^* \Gamma \quad\text{implies}\quad \wv' \cup \set{I},\wv' \models^* \Gamma
\end{gather*}
for every $I$ that belongs to some F15-equilibrium model of~$\Gamma$. As usual $\wv <_\Gamma \wv'$ stands for $\wv \leq_\Gamma \wv'$
and $\wv' \not\leq_\Gamma \wv$.
\end{definition}

\begin{observation}
If $I \in \wv$, then $\wv \cup \{ I \},\wv \models^* \Gamma$ iff $\wv,\wv \models^* \Gamma$ iff $\wv$ is a F15-equilibrium of~$\Gamma$.
\end{observation}

\begin{definition}[F15-world view]\label{def:F15.wv}
An epistemic interpretation~$\wv$ is called an \emph{F15-world view} of a theory~$\Gamma$
if it is an \mbox{F15-equilibrium} model of~$\Gamma$
and there is no other F15-equilibrium model~$\wv'$ such that $\wv \subset \wv'$ or $\wv <_\Gamma \wv'$.
\end{definition}

The following observation eases finding the F15-word views of many interesting programs.

\begin{observation}\label{obs:F15.wv}
If some theory~$\Gamma$ has a unique F15-equilibrium model~$\wv$,
then this is also its unique F15-world view.
\end{observation}

Let us now show that a program consisting of rule
\begin{gather}
p \leftarrow \K p
	\tag{\ref{es:self-supported}}
\end{gather}
has~$[\{ \, \}]$ as its unique F15-world view, as expected.
First, it is easy to see that~$[\{ \, \}]$ is an epistemic model of~\eqref{es:self-supported} and there is no~$\cI \prec [\{ \, \}]$.
Hence, this is an F15-equilibrium model.
On the other hand, if we consider~$\cI = \tuple{[\{p\}],h}$ with $h(\{p\}) = \emptyset$, then we can see that~$\cI \prec [\{p\}]$ and that~$\cI$ is an F15-model of~\eqref{es:self-supported}.
Therefore, $[\{p\}]$ is not an \mbox{F15-equilibrium} model.
In fact, we can check that~$[\{ \, \}]$ is the unique F15-equilibrium model of~$\{ \eqref{es:self-supported} \}$ and, from Observation~\ref{obs:F15.wv}, its unique F15-world view.

As mentioned earlier, a distinct characteristic of this semantics is that~$\M \fF$ cannot be understood as an abbreviation for~$\Not \K \Not \fF$.
In particular, the program consisting of the single rule
\begin{gather}
p \leftarrow \M p
	\tag{\ref{eq:self-supported.M}}
\end{gather}
has the unique world view~$[\{ \, \} ]$; while the program consisting of the single rule
\begin{gather}
p \leftarrow \Not \K \Not p
	\tag{\ref{eq:self-supported.M.translated}}
\end{gather}
has the unique world view~$[\{ p \} ]$.
The latter coincides with the semantics K15 and S16 presented above, but the former differs.
Beyond the difference on this particular example, this illustrates a major difference between the F15 semantics and the semantics K15 and S16: while the F15 semantics tries to reject self-supported beliefs through the operator~$\M$, both K15 and S16 force them.
Note that the G94 and the G11 semantics take an intermediate position with both programs having the same two world views: $[\{ p \} ]$ and~$[\{ \, \}]$.

Let us now show why~$[\{ \, \}]$ is the unique F15-world view of~\eqref{eq:self-supported.M}.
For this, note that interpretation~$\cI = \tuple{[\{ p \} ],h}$ with~$h(\{p\}) = \emptyset$ is an F15-model of~\eqref{eq:self-supported.M} and that it satisfies~$\cI \prec [\{ p \} ]$.
Hence, $[\{ p \} ]$ is neither an F15-equilibrium model nor an F15-world view of this program.
On the other hand, $[\{ \, \}]$ is trivially an F15-equilibrium model and, since there are no other F15-equilibrium models, it is the unique F15-world view.

Despite the fact that this semantics rejects more self-supported believes than previous semantics, it still manifests some self-supported believes
as can be illustrated using the following program.
\begin{gather}
p \tor q
\hspace{2cm}
p \leftarrow \K q
\hspace{2cm}
q \leftarrow \K p
\hspace{2cm}
\leftarrow \Not \K p
\label{theory:larger.set.of.worlds2}
\end{gather}
This program is the result of adding the constraint~$\leftarrow \Not \K p$
to the program~\eqref{theory:larger.set.of.worlds}.
This constraint is important to ensure that the program has a unique F15-equilibrium logic and, therefore, we can make use of Observation~\ref{obs:F15.wv}.
Note that~$[\{p\}, \{q\}]$ does not satisfy this constraint and, therefore, cannot be an F15-equilibrium model.
In fact, it is easy to check that~$[\{ p,q \}$ is the unique epistemic model of~\eqref{theory:larger.set.of.worlds2} and, thus, the only candidate to be an F15-equilibrium model.
To show that this is indeed an F15-equilibrium model, we need to check that there is no F15-model~$\cI = \tuple{ [\{p,q\}] ,h}$ with~$\cI \prec [\{p\}, \{q\}]$.
Note that such an interpretation must satisfy~$h(\{p,q\}) = \emptyset$, or~$h(\{p,q\}) = \{ p \}$ or or~$h(\{p,q\}) = \{ q \}$.
In the first case the interpretation does not satisfy the first disjunction;
in the other two cases, it fails to satisfy one of the other two rules.
Hence,~$[\{p\}, \{q\}]$ is the unique F15-equilibrium model and the unique F15-world view of this program.

It is also worth mentioning that the F15-semantics satisfies neither subjective constraint monotonicity (Property~\ref{property:constraint.monotonicity}) nor epistemic splitting (Property~\ref{property:epistemic.splitting}).
To illustrate this fact, consider the program consisting of the following two rules
\begin{gather}
p \tor q
\hspace{2cm}
\leftarrow \Not \K p,
	\tag{\ref{eq:leckah18b}}
\end{gather}
which has a unique F15-equilibrium model and a unique F15-world view~$\cset{ \set{ p } }$.
However, the program~$\{p \tor q\}$ is objective and has two stable models~$\set{a}$ and~$\set{b}$.
Thus, it has the unique world view~$\cset{ \set{p}, \set{q} }$.
Therefore adding the subjective constraint~$\leftarrow \Not  \K a$ produces a new world view, which violates the subjective constraint monotonicity property.

 \subsection{Foundedness property}
\label{sec:foundedness}

So far, we have seen that the search for self-supported-free beliefs was driven by a series of examples used to evaluate the different semantics for epistemic specifications.
\citeN{cafafa19a} presented a property called \emph{foundedness} that aims to capture the essence of this search in a general way.
This property is based on the notion of unfounded sets introduced by~\citeN{lerusc97a} for Answer Set Prolog.
Intuitively, an unfounded set is a collection of atoms that is not derivable from a given program and a fixed set of assumptions.

In order to formalize the foundedness property,
we need the following notation.
Given a ground rule~$r$,
the set $\Bodyrp(r)$ collects all explicit literals occurring in its positive body while $\Bodymp(r)$ collects all explicit literals occurring in positive subjective literals. 

\begin{definition}[Unfounded set]\label{def:unfoundedset}
Let $\Pi$ be a ground program and $\wv$ an epistemic interpretation.
An \emph{unfounded set} $\us$ with respect to $\Pi$ and $\wv$ is a non-empty set of pairs where, for each $\tuple{X,I} \in \us$, both~$X$ and~$I$ are sets of explict literals and there is no rule $r \in \Pi$ with $\Head(r) \cap X \neq \emptyset$ satisfying all of the following:
\begin{enumerate}
\item $\kdint{I}{\wv} \models \Body(r)$
	\label{item:1:def:unfounded}
\item $\Bodyrp(r) \cap X = \emptyset$
	\label{item:2:def:unfounded}
\item $(\Head(r) \setminus X) \cap I = \emptyset$
	\label{item:3:def:unfounded}
\item $\forall \tuple{X',I'} \in \us:$ $\Bodymp(r) \cap X' = \emptyset$  
	\label{item:4:def:unfounded}
\end{enumerate} 
\end{definition}

The definition is similar to unfounded sets for objective programs~\cite[Definition~3.1]{lerusc97a}.
In fact, the latter corresponds to the first three conditions above, except that $\kdint{I}{\wv}$ is used to check satisfaction of $\Body(r)$, as it may contain now subjective literals.
Intuitively, each $I$ represents some potential belief set and $X$ is some  set of atoms without a ``justifying'' rule.
In other words, there is no $r \in \Pi$ allowing a founded derivation of atoms in $X$.
A rule like that should have a true~$\Body(r)$ (condition~\ref{item:1:def:unfounded}) but not because of positive literals in~$X$ (condition~\ref{item:2:def:unfounded}) and is not used to derive other head atoms outside~$X$ (condition~\ref{item:3:def:unfounded}).
The novelty for the epistemic case is the addition of condition~\ref{item:4:def:unfounded}: to consider~$r$ a justifying rule, it is additionally required to not use any positive literal $\K a$ in the body such that atom $a$ also belongs to any of the unfounded components $X'$ in $\us$.

\begin{definition}[Founded world view]\label{def:unfounded}
Let $\Pi$ be a ground program and $\wv$ be an epistemic interpretation.
We say that $\wv$ is \emph{unfounded} if there is some 
unfounded set~$\us$ such that
every $\tuple{X,I} \in \us$
satisfies $I \in \wv$ and $X \cap I \neq \emptyset$.
We say that~$\wv$ is \emph{founded} otherwise.
\end{definition}

If we consider now the program in the introduction consisting of rule
\begin{gather}
p \leftarrow \K p
	\tag{\ref{es:self-supported}}
\end{gather}
we can observe that $\us=\sset{\tuple{\set{p},\set{p}}}$ makes $\sset{\emptyset, \set{p} }$ unfounded because~\eqref{es:self-supported} does not fulfill condition~4: we cannot derive atom $a$ from a rule that contains $a \in \Bodymp(r)$.
On the other hand, the other G94-world view, $\sset{\emptyset}$, is trivially founded.

\begin{property}[Foundedness]\label{property:unfounded-freedom}
A semantics $\cS$ satisfies \emph{foundedness} when all the $\cS$-world views of any ground program~$\Pi$ are founded.
\end{property}

As illustrated by the above example,
it is easy to see that the G94-semantics does not satisfy foundedness.
Note also that, as we have illustrated above, all the approaches presented so far present some self-supported beliefs.
Those self-supported beliefs are unfounded sets and, thus, we also can see that none of the approaches discussed so far satisfy this property.
Recall that the following program was used to illustrate the existence of self-supported beliefs in Sections~\ref{sec:gelfond2011} and~\ref{sec:kahl2015}:
\begin{gather}
a \tor b
\hspace{2cm}
a \leftarrow \K b
\hspace{2cm}
b \leftarrow \K a
\tag{\ref{theory:larger.set.of.worlds}}
\end{gather}
This program has two world views,
namely
$\sset{ \set{a}, \set{b} }$
and
$\sset{ \set{a , b} }$,
according to the G94, G11 and and K15 semantics.
World view~$\sset{ \set{a}, \set{b} }$ is founded because the first rule justifies both belief sets.
Note that
$\sset{\tuple{\set{a},\set{a}}}$ and $\sset{\tuple{\set{b},\set{b}}}$
are not unfounded sets.
However, for the world view~$\sset{ \set{a , b} }$, we have the unfounded set $\us' = \sset{ \tuple{\set{a},\set{a,b}}, \, \tuple{\set{b}, \set{a,b}}}$
which violates condition~\ref{item:3:def:unfounded} for the first rule and condition~\ref{item:4:def:unfounded} for the other two rules.
That is, the world view~$\sset{ \set{a , b} }$ is unfounded and, therefore, we can conclude that none of the G94, G11, and K15 semantics satisfy the foundedness property.

For the F15 and S16 semantics, we use the program~\eqref{theory:larger.set.of.worlds2} consisting of the rules of~\eqref{theory:larger.set.of.worlds} plus the constraint $\leftarrow \Not\K a$.
Note that adding subjective constraints to a program do not affect the existence of unfounded sets.
Thus, this example also shows that the F15 and S16 semantics do not satisfy this property either.

 \subsection{Cabalar, Fandinno \& Fari\~{n}as del Cerro 2019}\label{sec:cabalar2019}

Motivated by the fact that all previous approaches did not satisfy the foundedness property introduced above, \citeN{cafafa19a} presented a new semantics with this property in mind.
Another goal in designing this semantics was that it should precisely correspond to the G94-semantics when self-supported beliefs are not a problem.
This is the case for the class of programs that do not have positive dependencies through subjective literals.
Those programs are called \emph{epistemically tight}.
It is worth mentioning that the same ideas allow to obtain founded versions of the semantics mentioned above as we also illustrate below.

Technically, this semantics is an extension of Pearce's equilibrium logic with Moore's autoepistemic logic.
As a result,
its monotonic basis is based on a combination of the intermediate logic HT and the modal logic KD45.
In this sense, its monotonic basis is similar to the one presented in Section~\ref{sec:farinas2015}, but using modal logic KD45 instead of modal logic S5 used there.

A \emph{C19-epistemic interpretation}~${\wvb=\{\tuple{H_1,T_1}, \dots, \tuple{H_n,T_n}\}}$
is a non-empty set of pairs of propositional interpretations.
To each C19-epistemic interpretation, we associate a corresponding epistemic interpretation
${\wvb^t \eqdef \{T_1, \dots, T_n\}}$.
A \emph{C19-belief interpretation} $\cI$ is a pair $\cI=\kdint{\tuple{H,T}}{\wvb}$, or simply $\cI=\kdint{H,T}{\wvb}$, where $\wvb$ is a C19-epistemic interpretation and $\tuple{H,T}$ stands for the real world, possibly not in~$\wvb$.
A C19-belief interpretation
$\cI =\kdint{H,T}{\wvb}$ satisfies a formula~$\varphi$, written $\cI \models \varphi$, iff
\begin{itemize}
\item $\cI \models a$ iff $a \in H$, for any atom $a \in \At$,
\item $\cI \models \psi_1 \wedge \psi_2$ iff $\cI \models \psi_1$ and $\cI \models \psi_2$,
\item $\cI \models \psi_1 \vee \psi_2$ iff $\cI \models \psi_1$ or $\cI \models \psi_2$,
\item $\cI \models \psi_1 \to \psi_2$ iff both: 
(i) $\cI \not\models \psi_1$ or $\cI \models \psi_2$; and
(ii) $\kdint{T}{\wvb^t} \not\models \psi_1$ or $\kdint{T}{\wvb^t} \models \psi_2$,

\item $\cI \models \K \psi$ iff $\kdint{H_i,T_i}{\wvb} \models \psi$ for all $\tuple{H_i,T_i} \in \wvb$.

\end{itemize}
An interpretation $\kdint{H,T}{\wvb}$ is a \emph{C19-belief model} of a theory~$\Gamma$ iff $\kdint{H_i,T_i}{\wvb} \models \varphi$ for all $\tuple{H_i,T_i} \in \wvb \cup\{\tuple{H,T}\}$ and all $\varphi \in \Gamma$
 -- 
additionally, when $\tuple{H,T} \in \wvb$, we further say that $\wvb$ is a \emph{C19-epistemic model} of  $\Gamma$, abbreviated as $\wvb \models \Gamma$.

\begin{observation}
If~$\wvb$ is a C19-epistemic model of some theory~$\Gamma$, then~$\wvb^t$ is an epistemic model of~$\Gamma$.
\end{observation}

\begin{definition}\label{def:int.prec.views}
Given an epistemic interpretation~$\wv$ and a C19-epistemic interpretation~$\wvb$,
we write $\wvb \prec \wv$
if the following two conditions hold:
\begin{enumerate}
\item $\wvb^t = \wv$; and

\item for every $T \in \wv$,
there is some $\tuple{H,T} \in \wvb$,
with $H \subseteq T$.
\end{enumerate}
For a belief interpretation~$\tuple{\wv,I}$ and a C19-belief interpretation~$\tuple{\wvb,H,T}$, we write $\tuple{\wv,I} \preceq \tuple{\wvb,H,T}$ if $\wvb \preceq \wv$ and $I = T$.
We write $\wvb \prec \wv$ if $\wvb \preceq \wv$ and one of the following conditions hold:
\begin{enumerate}
\item there is $\tuple{H',T'} \in \wvb$ with $H' \subset T'$; or

\item $H \subset T$.
\end{enumerate}
\end{definition}

\begin{definition}
A belief interpretation $\kdint{I}{\wv}$ is said to be a \emph{C19-equilibrium model} of some theory~$\Gamma$, in symbols~$\kdint{I}{\wv} \models_{eq}$ iff its is a belief model of $\Gamma$ and there is no C19-belief model $\cI'$ of $\Gamma$ with $\cI' \prec \kdint{I}{\wv}$.
\end{definition}

As a final step, we impose a fixpoint condition to minimize the agent's knowledge as follows.
\begin{definition}[C19-world view]\label{def:au-eqmodel}
An epistemic interpretation~$\wv$ is called a \emph{C19-world view} of~$\Gamma$
if:
\begin{gather*}
\wv \ \ = \ \ \setm{ I }{ \kdint{I}{\wv } \models_{eq}  \Gamma  }
\end{gather*}
\end{definition}

\begin{theorem}[Main Theorem of the paper by~\citeNP{cafafa19a}]
\label{thm:g19.founded.G94}
Given any ground program~$\Pi$, its C19-world views coincide with its founded \mbox{G94-world} views.
\end{theorem}

This theorem does not only guarantee that all C19-world views are founded, but that the C19-world views are precisely those G94-world views that are founded.
As a result, it is easy to see that a program consisting of rule
\begin{gather}
p \leftarrow \K p
	\tag{\ref{es:self-supported}}
\end{gather}
has~$[\{ \, \}]$ as its unique C19-world view as expected, because this is the only G94-world view that is founded.
Similarly, a program consisting of rules
\begin{gather}
p \tor q
\hspace{2cm}
p \leftarrow \K q
\hspace{2cm}
q \leftarrow \K p
\tag{\ref{theory:larger.set.of.worlds}}
\end{gather}
has a unique founded G94-world views, namely $\sset{ \set{p}, \set{q} }$, which is thus its only C19-world view.
The same applies to the program consisting of rules
\begin{gather}
p \tor q
\hspace{2cm}
p \leftarrow \K q
\hspace{2cm}
q \leftarrow \K p
\hspace{2cm}
\leftarrow \Not \K p
\tag{\ref{theory:larger.set.of.worlds2}}
\end{gather}
which has the same unique C19-world view.

Given that Theorem~\ref{thm:g19.founded.G94} states that the C19-world views are the founded G94-world views, we may expect that these two semantics coincide for programs where self-supported beliefs are not an issue.
As mentioned above, this class of programs is called epistemically tight and consist of programs that do not have positive dependencies through subjective literals.
Formally, the positive epistemic dependence relation among atoms in a program $\Pi$ is defined so that $dep^+(a,b)$ is true iff there is any rule \mbox{$r \in \Pi$} such that
$a \in \Head(r) \cup \Bodyr(r)$ and $b \in \Bodymp(r)$.

\begin{definition}[Epistemically tight program]
We say that an epistemic program $\Pi$ is \emph{epistemically tight} if we can assign an integer mapping
\mbox{$\lambda: \At \longrightarrow \mathbb{N}$} to each atom such that
\begin{enumerate}
    \item \mbox{$\lambda(a) = \lambda(b)$} for any rule $r \in \Pi$ and atoms $a,b \in (\Atoms(r)\setminus \Bodym(r))$,
    \item \mbox{$\lambda(a)>\lambda(b)$} for any pair of atoms $a,b$ satisfying $dep^+(a,b)$.
\end{enumerate}
\end{definition}

\begin{theorem}[Theorem~8 in the paper by~\citeNP{fandinno19a}]\label{thm:tight.coincide}
C19- and \mbox{G94-world} views coincide for epistemically tight programs.
\end{theorem}

This class of programs includes for instance the eligibility program introduced in Example~\ref{ex:college} (Section~\ref{sec:splitting}).
It also includes the programs corresponding to the rules~$\eqref{prg:g11.no.epistemic.splitting}$, $ \eqref{prg:g11.no.epistemic.splitting.no.constraint} $ and~$\eqref{eq:leckah18b}$ discussed above.

Theorems~\ref{thm:g19.founded.G94} and~\ref{thm:tight.coincide} also provide means for using tools to compute G94-world views as a means to compute C19-world views.
If the program is epistemically tight, we can just use a tool for the G94 semantics directly.
Otherwise, we can use a tool for the G94 semantics to compute a candidate and then check whether this candidate is founded.

Interestingly, the G94 semantics can also be characterized as a particular class of theories under the C19 semantics as illustrated next.
Let~$\KEM$ be the set containing the following form of \emph{exclude middle axiom}
\begin{gather*}
\K\,(\ell \vee \Not \ell)
\end{gather*}
for every explicit literal~$\ell$.

\begin{proposition}[Proposition~5 by~\protect\citeNP{cafafa20a}]\label{thm:g94.weak}
The G94-world views of any theory $\Gamma$ coincide precisely with the C19-world views of
$\Gamma \cup \KEM$.
\end{proposition}

In light of these results, we can understand the C19 semantics as a founded version of the G94 semantics.
Interestingly, these results allow for providing founded versions for all of the semantics presented above.
Recall from Section~\ref{sec:kahl2015} that the K15-world views of any program~$\Pi$ can be characterized as the G94-world views of program~$\Pi^K$ (Corollary~\ref{cor:reflexive.embedding2}).
Using this translation we can obtain a founded version of the K15 semantics.

\begin{definition}
An epistemic interpretation~$\wv$ is called a \emph{FK15-world view} of~$\Gamma$ if
$\wv$ is a C19-world view of~$\Gamma^K$.
\end{definition}

A founded version of the G11 semantics can be obtained in a similar way, by providing a variation of the translation $\cdot^K$ that only affects positive occurrences of the operator~$\K$.
Furthermore, since S16-world views can be defined in terms of K15-world views, we can immediately get a a founded version of the S15 semantics by replacing in the definition of S16-world views each occurrence of K15 by FK15.
A founded version of the F15 semantics is slightly more involved and we refer to the paper by~\citeN{fandinno19a} for more details.

 \subsection{The state of the search}
\label{sec:semantics.state}

So far in this section, we have reviewed the major approaches that have addressed the issue of the existence of self-supported beliefs.
While doing so, we have also reviewed how these approaches behave with respect to several properties inspired by properties satisfied by the stable model semantics.
Table~\ref{table:summary} is taken from the paper by~\citeN{fandinno19a} and summarizes the known results for these semantics with respect to those properties.
\begin{table}
\centering
\begin{tabular*}{13.25cm}{  @{}p{11pc} @{\hskip1cm}  @{}p{3.2pc} @{}p{3.2pc} @{}p{3.2pc} @{}p{3.2pc} @{}p{3.2pc} @{}p{3.2pc} @{}p{3.2pc} @{}p{3.2pc} @{}p{3.2pc} }
\hline\hline
&G94 & G11 
& F15 
& K15 & S16
& C19
\\\hline
Supra-S5& 
\checkmark & \checkmark 
&\checkmark
& \checkmark & \checkmark 
& \checkmark
\\\hline
Supra-ASP& 
\checkmark & \checkmark 
& \checkmark
& \checkmark & \checkmark
& \checkmark 
\\\hline
Sub. constraint monotonicity&
\checkmark & \checkmark 
& 
& &
& \checkmark
\\\hline
Splitting
&\checkmark &  
& 
& &
& \checkmark
\\\hline
Foundedness&
& & & &
& \checkmark
\\\hline\hline
\end{tabular*}
	\caption{Summary of properties in different semantics.}
	\label{table:summary}
\end{table}
As we can see, C19 is the only one that satisfies foundedness
and, as illustrated in Section~\ref{sec:cabalar2019}, we can use this semantics to construct founded versions of all other semantics.
In fact, C19 can be considered as the ``founded version'' of G94.

Another interesting fact is that only G94 and C19 satisfy the epistemic splitting property.
Founded versions of the other semantics do not satisfy epistemic splitting either.
In fact, this property seems to be tightly connected with the \emph{non-reflexivity} of these two semantics.
We say that a semantics~$\cS$ is \emph{reflexive} when the~$\cS$-world views of~$\Gamma$ and~${\Gamma \cup \{ p \leftarrow \K p  \mid p \in \At \}}$ coincide for every possible epistemic theory~$\Gamma$.
It is easy to see that G11, F15, K15 and S16 are all reflexive, while G94 and C19 are not.
As a result of this trade-off, we can find examples that appear to have self-supported beliefs in all semantics.
Take for instance, the program consisting of the following rules:
\begin{align}
p &\leftarrow \K p
	\tag{\ref{es:self-supported}}
\\
p &\tor q
\\
s &\leftarrow \K p
	\label{eq:challenge:3}
\\
&\leftarrow \Not s
\end{align}
This program is the result of adding rule~\eqref{es:self-supported} to~\eqref{prg:g11.no.epistemic.splitting} and has the unique world view~$[\{p,s\}]$ according to all semantics.
As mentioned above, for reflexive semantics, rule~\eqref{es:self-supported} is redundant and the world views of this program coincide with the world views of~\eqref{prg:g11.no.epistemic.splitting}.
We already analyzed this program in previous sections and showed that~$[\{p,s\}]$ is its unique G11- and K15-world view.
As mentioned in Section~\ref{sec:shen2017}, being a unique K15-world view immediately implies that this is also the unique S16-world view.
It also can be checked that this is the unique F15-world view.
Note that the first rule supports that~$p$ may be true in some answer sets, but does not support that~$p$ is true in all answer sets.
Therefore, the body of~\eqref{eq:challenge:3} is not supported and neither should be~$s$.
Thus, both~$p$ and~$s$ are unsupported beliefs.
Recall that the reason why these semantics produce this unintended world view is related with their failure to satisfy the epistemic splitting property.
In fact, program~\eqref{prg:g11.no.epistemic.splitting} has no world view according to any semantics that satisfy epistemic splitting, like~G94 or C19.
However, as a result of the non-reflexivity of these semantics, adding~\eqref{es:self-supported} produces the world view~$[\{p,s\}]$.
This also seems unjustified as the body of~\eqref{es:self-supported} still lacks justification.
More research is necessary to understand the behavior of this kind of programs and whether these apparently unsupported beliefs can be avoided.

 \section{Relation to autoepistemic logics}
\label{sec:ael}

It is well-known that, for any ground program that includes a choice rule for all its atoms, its stable models coincide with the classical models of the program understood as a propositional theory.
In this section, we show that there is a similar relation between some of the approaches for epistemic specifications and some autoepistemic logics.
Recall that autoepistemic logics are nonmonotonic logics for
modeling the beliefs of ideally rational agents who reflect on their own
beliefs.

The first and most influential of these logics was the one introduced by~\citeN{moore85}.
The language of autoepistemic logic is that of ordinary propositional logic, augmented by a modal operator~$\LK$.
Formulas of the form~$\LK \varphi$ can be read as ``$\varphi$ is believed.''
In order to make the comparison with epistemic specifications easier,
we replace the modal operator~$\LK$ by~$\K$.
With this notation at hand,
we can say that a set of formulas~$E$ is a \emph{stable expansion}
of a theory~$\Gamma$ (whose only modal operator is~$\K$)
if~$E$ is the set of all consequences (in the sense of classical
propositional logic) of theory
\begin{gather*}
\Gamma \cup \{ \K \varphi \mid \varphi \in E \} \cup \{ \Not \K \varphi  \mid \varphi \notin E \}
\end{gather*}
We can easily extend this definition to arbitrary theories by assuming
that operators~$\M$ and~$\eNot$ are shorthands as stated in Table~\ref{table:interdefinition.opeartors}.

Moore soon realized that autoepistemic logic can be characterized in terms of the consequences of modal logic KD45 instead of classical
propositional logic.
Later,
\mbox{\citeN{schwarz92}}
showed that it is also possible to characterize autoepistemic logic as a particular class of minimal epistemic models.
This characterization can be rewritten in form of a fixpoint similar to Definition~\ref{def:au-eqmodel}.
We name those epistemic models as M85-world views by analogy with epistemic specifications.

\begin{samepage}
\begin{definition}[\Mefwv view]
We say that an epistemic interpretation~$\wv$ is an \emph{\Mefwv view} of some theory~$\Gamma$ when it satisfies the following fixpoint condition:
\begin{gather*}
\wv \quad = \quad \{ \, I \mid \tuple{\wv,I} \models \Gamma \, \}
\end{gather*}
\end{definition}
\end{samepage}

\begin{proposition}[Proposition~4.1 in the paper by~\citeNP{schwarz92}]
\label{prop:moore85.models}
Let~$\Gamma$ be a theory, $\wv$ be an epistemic interpretation
and~$E = \{ \varphi \mid \wv \models \varphi \}$ be the set of formulas satisfied by~$\wv$.
Then,
$\wv$ is a \Mefwv view of~$\Gamma$
iff
$E$ is a stable expansion of~$\Gamma$.
\end{proposition}

Proposition~\ref{prop:moore85.models} provides a semantic characterization of Moore's autoepistemic logic.
This semantic characterization is similar to the definition of C19-world views.
In fact, the definition of C19-world views is obtained by replacing the satisfaction in modal logic KD45 by equilibrium satisfaction.
That is, by replacing~$\tuple{\wv,I} \models \Gamma$ by $ \tuple{\wv,I} \models_{eq} \Gamma$.
Similarly, the G94 semantics can also be characterized as a similar fixpoint where the equilibrium condition is weakened (see Appendix~A in~\citeNP{fandinno19a}).
This allows us to show that autoepistemic logic can be captured by a particular class of theories under the G94 or C19 semantics.
Let~$\EM$ be the set containing the \emph{excluded middle axiom}
\begin{gather*}
\ell \vee \Not \ell
\end{gather*}
for every explicit literal~$\ell$.

\begin{proposition}[Theorem~1 in the paper by~\protect\citeNP{cafafa19a}]\label{thm:g94.vs.m85}
The M85-world views of any theory $\Gamma$ coincide precisely with the G94-world views of
$\Gamma \cup \EM$.
\end{proposition}

\begin{proposition}\label{thm:autoepistemic.weak}
The M85-world views of any theory $\Gamma$ coincide precisely with the C19-world views of
$\Gamma \cup \EM$.
\end{proposition}

\begin{proof}
\phantom{iff} $\wv$ is a M85-world view of~$\Gamma$
\\iff $\wv$ is a G94-world view of~$\Gamma \cup \EM$ \hfill(Proposition~\ref{thm:g94.vs.m85})
\\iff $\wv$ is a C19-world view of~$\Gamma \cup \EM \cup \KEM$ \hfill(Proposition~\ref{thm:g94.weak})
\\iff $\wv$ is a C19-world view of~$\Gamma \cup \KEM$.
\\
For the last equivalence,
just note that any belief model of~$\EM$
is also a belief model of~$\KEM$.\hfill
\end{proof}

As mentioned above, the stable models of any theory that includes the excluded middle axiom for all atoms coincide with its models in classical propositional logic.
Propositions~\ref{thm:g94.vs.m85} and~\ref{thm:autoepistemic.weak} show that a similar relation exists between the G94 and C19 semantics for epistemic specifications and Moore's autoepistemic logic.
In this sense, we can consider these semantics the ``stable'' versions of Moore's autoepistemic logic.

Besides Moore's autoepistemic logic, several alternatives have been studied in the literature~\cite{konolige88a,martru89,niemela91a,schwarz91a}; most of them also motivated by the existence of self-supported beliefs in this logic.
In particular, one of this alternatives, called \emph{reflexive autoepistemic logic}~\cite{schwarz91a}, is closely related to the K15 semantics for epistemic specifications.

Formally,
a set of formulas~$E$ is a \emph{reflexive expansion}
of a theory~$\Gamma$ (whose only modal operator is~$\K$)
if~$E$ is the set of all consequences (in the sense of classical
propositional logic) of the theory
\begin{gather*}
\Gamma \cup \{ \varphi \leftrightarrow \K \varphi \mid \varphi \in E \} \cup \{ \Not \K \varphi  \mid \varphi \notin E \}
\end{gather*}
Alternatively,
reflexive expansions
can be characterized as Moore's stable expansions of its reflexive embedding~$(\,\cdot\,)^B$ (see Section~\ref{sec:kahl2015}).

\begin{proposition}
A set of formulas~$E$ is a reflexive expansion of some theory~$\Gamma$
iff it is a stable expansion of the theory~$\Gamma^B$.
\end{proposition}

\begin{proof}
Directly from Theorem~10.30 and~10.48 by~\citeN{marek93}.
\end{proof}

Using this result, we can semantically characterize reflexive autoepistemic logic as follows.\footnote{Alternatively, S92-world views can be characterized using the modal logic SW5 instead of the reflexive embedding (see~\citeNP{schwarz92}).}

\begin{samepage}
\begin{definition}[\Sntwv view]
We say that an epistemic interpretation~$\wv$ is an \emph{\Sntwv view} of some theory~$\Gamma$ when it satisfies the following fixpoint condition:
\begin{gather*}
\wv \quad = \quad \{ \, I \mid \tuple{\wv,I} \models \Gamma^B \, \}
\end{gather*}
\end{definition}
\end{samepage}

\begin{corollary}
Let~$\Gamma$ be a theory, $\wv$ be an epistemic interpretation
and~$E = \{ \varphi \mid \wv \models \varphi \}$ be the set of formulas satisfied by~$\wv$.
Then,
$E$ is a reflexive expansion of~$\Gamma$
iff
$\wv$ is a \Sntwv view of~$\Gamma$.
\end{corollary}

The following result shows that K15 can be considered the ``stable'' version of reflexive autoepistemic logic.

\begin{proposition}\label{thm:reflexive.autoepistemic.weak}
The S92-world views of any theory $\Gamma$ coincide precisely with the K15-world views of
$\Gamma \cup \EM$.
\end{proposition}

\begin{proof*}
\phantom{iff} $\wv$ is a S92-world view of~$\Gamma$
\\iff $\wv$ is a M85-world view of~$\Gamma^B$ \hfill(By definition)
\\iff $\wv$ is a G94-world view of~$(\Gamma  \cup \EM)^B$ \hfill(Proposition~\ref{thm:g94.vs.m85})
\\iff $\wv$ is a K15-world view of~$\Gamma \cup \EM$ \hfill(Proposition~\ref{prop:reflexive.embedding})
\end{proof*}
 \section{Systems for computing world views}
\label{sec:systems}

Currently, there are several systems to compute the world views of a epistemic logic program: 
\mbox{\emph{ESmodels}}~\cite{ZhangZ14}, 
\mbox{\emph{Wviews}}~\cite{Kelly2007,Kelly2018},
\mbox{\emph{ELPS}}~\cite{BalaiK14},
\mbox{\emph{GISolver}}~\cite{ZhangWZ2015},
\mbox{\emph{ELPsolve}}~\cite{LeclercK16},
\mbox{\emph{EP-ASP}}~\cite{LeS2017},
\mbox{\emph{PelpSolver}}~\cite{ZhangZ2017},
\mbox{\emph{EHEX}}~\cite{Strasser2018}, 
\mbox{\emph{selp}}~\cite{BichlerMW20},
\mbox{\emph{eclingo}}~\cite{cafagarosc20aURL}.
\begin{table}[h]
    \footnotesize
    \setlength{\tabcolsep}{3pt}
    \begin{tabular}{p{25mm} ll p{15mm} l p{22mm} }
        \hline\hline
        Solver & Year & Semantics & Underlying ASP solver & Imp. Lang & Available Form
        \\
        \hline\hline
        ELMO & 1994 & G94 & dlv & Prolog & n/a (in thesis)
        \\
        \hline
        sismodels & 1994 & G94 & claspD & C++ &n/a
        \\
        \hline
        Wviews & 2007 & G94 & clingo & C++ & Windows binary
        \\
        \hline
        Esmodels & 2013 & G11 & clingo & (unknown) & Windows binary
        \\
        \hline
        ELPS & 2014 & K15 & Java & clingo & source + binary
        \\
        \hline
        GISolver & 2015 & K15 & clingo & (unknown) & Windows binary
        \\
        \hline
        ELPsolve & 2016 & K15/S16 & clingo & C++ & binary only
        \\
        \hline
        Wviews2 & 2017 & G94 & Python & clingo & Windows binary
        \\
        \hline
        EP-ASP & 2017 & K15/S16 & clingo & Python + ASP & Windows binary
        \\
        \hline
        PelpSolver & 2017 & S16 & clingo & Java & Windows binary
        \\
        \hline
        ELPsolve2 & 2017 & S16 & clingo & C++ & currently not for public release
        \\
        \hline
        EHEX & 2018 & S16 & clingo & Python & source
        \\
        \hline
        selp & 2018 & S16 & clingo & Python & source
        \\
        \hline
        eclingo & 2020 & G94 & clingo & Python & source
        \\
        \hline
    \end{tabular}
    \caption{List of solvers for computing the world views of epistemic logic programs.
}
    \label{table:solvers}
\end{table}
A recent survey can be found in the paper by~\citeN{leckah18}.
For the sake of completeness, Table~\ref{table:solvers} briefly summarizes some of the characteristics of these solvers discussed in this survey with the addition of the recently presented~\texttt{eclingo}.
It is worth mentioning that~\mbox{\citeN{HecherMW20}} recently presented a new dynamic programming algorithm for computing the world views of an epistemic logic program.
This algorithm bounds the number of calls necessary to the underlying solver for Answer Set Prolog by using the treewidth of the program.
The authors have communicated to us that they are currently working on an implementation of this algorithm.

\section{Conclusions and Challenges}
\label{sec:conclusions}
The paper presents a review of the development of the Theory of Epistemic
Specifications.
The language was introduced in the early nineties with the goal of expanding Answer Set Prolog
with means of reasoning with incomplete information in the presence of
multiple answer
sets. It belongs to the body of work aimed at better understanding and
automating common
sense reasoning by developing formal knowledge representation languages
and reasoning algorithms and learning how they can be used to
take a simple
story, encode it on a machine in some way, and then test to see if the machine
can correctly answer questions that a human can answer.
Judea Pearl refers to such work as an attempt to pass what he
calls Mini-Turing Test (Pearl and Mackenzie 2018).
Even though we are still very far from passing the test, Epistemic
Specifications help to make a small step in the right direction.
Their use allows us to expand the collection of stories one can successfully deal
with. Section~\ref{sec:applications} shows examples of such stories. Unfortunately the progress
was limited to stories whose formalization did not require recursion
through modal operators. If such recursion were required the original semantics
produced counter-intuitive results.
For a long time this line of research has been put on the
back burner but in the last decade we have seen a renewed interest in the subject
and there has been a substantial progress in the understanding
of the language. We described various approaches to defining the semantics,
relationships between them, and their properties. In addition, our understanding was deepened
by discoveries of important connections between epistemic specifications and (both monotonic and non-monotonic) modal logics. Despite this progress we still have a number of important open problems
to solve.

We need to gain more experience in using epistemic specifications for knowledge representation. This
will allow us to learn if the expressive power of the language is sufficient for
its original purpose. In particular, it remains to be seen if the language is fully
suited for representing various forms of the Closed World Assumption -- one of its original goals.
This is also necessary for the development of methodologies for the use of epistemic specifications.

More work is needed to further develop the mathematical theory of epistemic specifications. 
Most formal results are only specified for the propositional fragment of the language while the use of quantifiers for knowledge representation seems essential.
We also need to
check whether the theory of Answer Set Prolog modules can be adapted to work in epistemic specifications, study various forms
of equivalence between epistemic theories (some preliminary work on strong and uniform equivalence has been reported in the papers by~\citeANP{famowo19a}~\citeyearNP{famowo19a,famowo19b}; and~\citeNP{sufahe20}), find conditions for existence and/or
uniqueness of world views, develop more efficient reasoning
algorithms, to name just a few. 
Even though there are several prototype solvers for epistemic specifications, they efficiency and usability still requires substantial work to be applicable in education and/or efficient for industrial
applications.

It may be important to further expand the language of epistemic specifications. Inclusion of
aggregates, sets, numerical constraints can be guided by the corresponding work
which has already been done in Answer Set Prolog.
But making epistemic specifications suitable for serving multiple agents
or deal with probabilistic reasoning may prove to be formidable problems.

\bibliographystyle{include/tlp/acmtrans}

\appendix
\section{Quantified Equilibrium Logic with Explicit Negation}
\label{sec:qel}

In this section we review the semantics of quantified equilibrium logic~\cite{peaval06a} and extend it with explicit negation~\cite{agcafapepevi19b}.
We limit the exposition here to the language presented in Section~\ref{sec:theories}, that is, we do not consider function symbols in our language and assume that the domain consists exactly of the set of ground terms.

Then, an \emph{HT-interpretation} is a pair~$\tuple{H,T}$ where both~$H$ and~$T$ are interpretations as defined in Section~\ref{sec:theories.semantics}.
As we did with belief interpretations,
we write that~${\tuple{H,T} \models \fF}$, to represent that an HT-interpretation
${\tuple{H,T}}$ \emph{satisfies} an objective formula~$\fF$
and~${\tuple{H,T} \falsif \fF}$ to represent that a HT-interpretation $\tuple{H,T}$ \emph{falsifies} an objective formula~$\fF$.
Note that the ambiguity is removed from the interpretation.
These two relation are defined according to the following recursive conditions:
\begin{enumerate}
\item $\tuple{H,T} \not\models \bot$;
\item $\tuple{H,T} \models \top$;
\item $\tuple{H,T} \models a$ if $a \in H$, for any atom $a \in \At$:

\item $\tuple{H,T} \models \fF \wedge \fG$ if $\tuple{H,T} \models \fF$ and $\tuple{H,T} \models \fG$;

\item $\tuple{H,T} \models \fF \vee \fG$ if $\tuple{H,T} \models \fF$ or $\tuple{H,T} \models \fG$;

\item $\tuple{H,T} \models \fF \leftarrow \fG$ if both
$\tuple{H,T} \models \fF$ or $\tuple{H,T} \not\models \fG$, and
\\
\phantom{$\tuple{H,T} \models \fF \leftarrow \fG$ if both} $\tuple{T,T} \models \fF$ or $\tuple{T,T} \not\models \fG$;

\item $\tuple{H,T} \models \exists x \, \fF(x)$ if $\tuple{H,T} \models \fF(t)$ for some ground term~$t$;

\item $\tuple{H,T} \models \forall x \, \fF(x)$ if $\tuple{H,T} \models \fF(t)$ for all ground terms~$t$;

\item $\tuple{H,T} \models \sneg \fF$ if $\tuple{H,T} \falsif \fF$;

\vspace{5pt}

\item $\tuple{H,T} \falsif \bot$;
\item $\tuple{H,T} \not\falsif \top$;
\item $\tuple{H,T} \falsif a$ if $\sneg a \in H$, for any atom $a \in \At$:

\item $\tuple{H,T} \falsif \fF \wedge \fG$ if $\tuple{H,T} \falsif \fF$ or $\tuple{H,T} \falsif \fG$;

\item $\tuple{H,T} \falsif \fF \vee \fG$ if $\tuple{H,T} \falsif \fF$ and $\tuple{H,T} \falsif \fG$;

\item $\tuple{H,T} \falsif \fF \leftarrow \fG$ if 
$\tuple{H,T}  \falsif \fF$ and $\tuple{T,T} \models \fG$

\item $\tuple{H,T} \falsif \exists x \, \fF(x)$ if $\tuple{H,T} \falsif \fF(t)$ for all ground terms~$t$;

\item $\tuple{H,T} \falsif \forall x \, \fF(x)$ if $\tuple{H,T} \falsif \fF(t)$ for some ground term~$t$; and

\item $\tuple{H,T} \falsif \sneg \fF$ if $\tuple{H,T} \models \fF$.
\end{enumerate}
An HT-interpretation~$\tuple{H,T}$ that satisfies an objective formula is called a~\emph{HT-model} or just a \emph{model} when it is clear by the context.
Similarly, An HT-interpretation~$\tuple{H,T}$ is a model of some objective theory if it is a model of all its formulas.

\begin{definition}[Equilibrium model and answer set]
An HT-interpretation of the form~$\tuple{T,T}$ is an \emph{equilibrium model} of an objective theory~$\Gamma$ if $\tuple{T,T}$ is a model of~$\Gamma$
and there is no other model $\tuple{H,T}$ of $\Gamma$ with~$H \subset T$.

An interpretation~$I$ is an \emph{answer set} (or \emph{stable model}) of an objective theory~$\Gamma$
if~$\tuple{I,I}$ is an equilibrium model of~$\Gamma$.
By~$\SM[\Gamma]$ we denote the set of all answer set of~$\Gamma$.
\end{definition}
 
\end{document}